\documentclass[a4paper,UKenglish]{lipics}
 
\usepackage{microtype}

\title{Fast nonlinear embeddings via structured matrices}
\titlerunning{Fast nonlinear embeddings via structured matrices} 

\author[1]{Krzysztof Choromanski}
\author[2]{Francois Fagan}
\affil[1]{Google Research, New York, NY 10011, USA, \texttt{kchoro@google.com}}
\affil[2]{Department of IEOR, Columbia University, New York, NY 10027, USA, \texttt{ff2316@columbia.edu}}
\authorrunning{K. Choromanski \& F. Fagan} 

\Copyright{Author}



\begin{document}

\maketitle

\begin{abstract}
We present a new paradigm for speeding up randomized computations of several frequently used functions in machine learning. In particular, our paradigm can be applied for improving
computations of kernels based on random embeddings. Above that, the presented framework covers multivariate randomized functions. As a byproduct, we propose an algorithmic
approach that also leads to a significant reduction of space complexity. Our method is based
on careful recycling of Gaussian vectors into structured matrices that share properties of fully random
matrices. The quality of the proposed structured approach follows from combinatorial properties of the
graphs encoding correlations between rows of these structured matrices. Our framework covers as special cases
already known structured approaches such as the Fast Johnson-Lindenstrauss Transform, but is much more general since
it can be applied also to highly nonlinear embeddings. We provide strong concentration results showing the quality
of the presented paradigm.
 \end{abstract}

\section{Introduction}

Dimensionality reduction techniques and nonlinear embeddings based on random projections is a well-established field of machine
learning. It is built on the surprising observation that the relationship between points in a high-dimensional space might be approximately
reconstructed from a relatively small number of their independent random projections. This relationship might be encoded by the
standard Euclidean distance, as is the case for the Johnson-Lindenstrauss Transform \cite{krahmer2011new}, or a nonlinear function
such as kernel similarity measure \cite{rahimi}. 
These techniques are applied in compression and information retrieval \cite{ailon2006approximate, andoni2006near,dasgupta2011fast,gong2012angular},
compressed sensing due to the related restricted isometry properties \cite{ailon2014fast, bourgain2011explicit,dehghan2015restricted},
quantization \cite{boufounos2015representation,jacques2013quantized} and many more.
One particularly compelling application involves random feature selection techniques that were successfully used for
large-scale kernel computation \cite{huang,rahimi,smola_book,vedaldi,xie}.
The randomized procedure for computing many of the kernels' similarity measures/distances considered in that setting (including Euclidean distance, angular similarity kernels,
arc-cosine kernels and Gaussian kernels) is based on using first a Gaussian random
mapping and then applying pointwise nonlinear mappings (thus the computations mimic these in the neural network setting, but the linear projection is not learned).
This is the area of our interest in this paper.

Recently it was observed that for some of these procedures unstructured random matrices can be replaced by their structured counterparts
and the quality of the embedding does not change much. A structured matrix $\textbf{A} \in \mathbb{R}^{m \times n}$ uses $t < mn$ random gaussian
variables and distributes them in a way that approximately preserves several properties of the completely random version. The importance of the
structured approach lies in the fact that it usually provides speedups of matrix-vector multiplication --- a key computational block of the unstructured
variant --- by the exploitation of the matrix structure (for instance for Gaussian circulant matrices one may use the Fast Fourier Transform to reduce the computational time
from $O(mn)$ to $O(n\log(m))$). Furthermore, it gives a space complexity reduction since structured matrices can be stored in subquadratic or even linear space.
A structured approach to linear embeddings, called the Fast Johnson-Lindenstrauss Transform, is itself a subject of vast volume of research results that use different approaches involving: Hadamard matrices
with Fast Fourier Transforms and sparse matrices \cite{ailon2006approximate, ailon2013almost, ailon2014fast, bingham2001random, krahmer2011new},
binary and sparse matrices \cite{achlioptas2003database,dasgupta2010sparse, kane2010derandomized,kane2014sparser,matouvsek2008variants,nelson2013osnap}, Lean Walsh Transform \cite{liberty2008dense}, circulant matrices
\cite{hinrichs2011johnson,krahmer2014suprema,rauhut2009circulant,rauhut2012restricted,romberg2009compressive,vybiral2011variant, zhang2013new} and others. Here no nonlinear mappings are used. In this paper we are interested mainly in structured nonlinear embeddings.
Not much is known in that area. All known results target very specific kernels, such as the angular similarity kernel \cite{choromanska2015binary, gong2012angular,yi2015binary,yu2015binary} and Gaussian kernel \cite{le2013fastfood}, and/or use a fixed budget of randomness to construct a structured matrix 
(all above but \cite{choromanska2015binary}) since the construction of the structured matrix is very rigid. 

This is all relevant to neural networks which have matrix-vector multiplication and nonlinear transformations at their core. Structured matrices have been used in neural networks to speed up matrix-vector computation, decrease storage and sharply reduce the number of training parameters without much affecting performance \cite{cheng2015exploration, moczulski2015acdc,sindhwani2015structured,yang2014deep}. Random weight matrices eliminate training weight matrices altogether \cite{choromanska2015binary,saxe2011random} and provide a pathway for analyzing neural networks \cite{giryes2015deep}. Explicitly integrating out random weights produces new kernel-based algorithms, such as the arc-cosine kernel \cite{cho2009kernel}.

Even though there is some agreement which structured approaches may work in 
practice for specific applications, general characteristics of structured matrices producing high quality embeddings for general nonlinear mappings as well as the underlying theoretical 
explanation was not known.  
In this paper we propose such a general framework that covers as special cases most of the existing structured mechanisms and can be automatically adjusted to different ``budgets
of randomness'' used for structured matrices. The latter property enables us to smoothly transition from the completely unstructured setting, where the quality guarantees are
stronger but computational cost is larger, to the structured setting, where we can still prove quality results but the computations are sped up and space 
complexity is drastically reduced. At the same time we show an intriguing connection between guarantees regarding the quality of the produced structured nonlinear embeddings 
and combinatorial properties of some graphs associated with the structured models and encoding in a compact form correlations between different rows of the structured matrix. 

The randomized function $\Lambda_{f} : \mathbb{R}^{k} \rightarrow \mathbb{R}$ for which we propose a structured computational model takes as an input $k$ vectors $\textbf{v}^{1},...,\textbf{v}^{k} \in \mathbb{R}^{n}$.
Thus the model is general enough to handle relations involving more than $k=2$ vectors. Each vector is preprocessed by multiplying it with a 
Gaussian matrix $\textbf{R}=[\textbf{r}^{1},...,\textbf{r}^{m}]^\top$, where $\textbf{r}^{i}$ stands for the $i^{th}$ row, and pointwise nonlinear mapping $f$ following it. 
We then apply
another mapping $\beta:\mathbb{R}^{k} \rightarrow \mathbb{R}$ separately on each dimension (in the context of kernel computations mapping $\beta$ is simply a 
product of its arguments) and finally agglomerate the results for all $m$ dimensions by applying another mapping $\Psi:\mathbb{R}^{m} \rightarrow \mathbb{R}$. 
Functions $\Lambda_{f}$ defined in such a way, although they may look complicated at first glance, encode all the distance/kernels' similarity measures that we have 
mentioned so far. We show that the proposed general structured approach enables us to get strong concentration results regarding the computed structured
approximation of $\Lambda_{f}$. 

Presented structured approach was considered in \cite{choromanski_sindhwani2016}, but only for the computations of specific kernels
(and the results heavily relied on the properties of these kernels). Our concentration results are also much sharper, since we do not rely
on the moments method and thus cover datasets of sizes superpolynomial in $m$. In \cite{choromanski_sindhwani2016} the authors empirically
verify the use of the multi-block ``Toeplitz-like'' matrices for feature set expansion which is not our focus here since we reduce dimensionality.

This paper is organized as follows:
\begin{itemize}
\item in Section 2 we introduce our structured mechanism and propose an algorithm using it for fast computations of nonlinear embeddings,
 \item in Section 3 we present all theoretical results,
 \item in the Appendix we prove all theoretical results that were not proved in the main body of the paper.
\end{itemize}

\section{Structured mechanism for fast nonlinear embeddings}
\label{sec:structured_one}

\subsection{Problem formulation}

We consider in this paper functions of the form: 
\begin{equation}
\label{main_function}
\Lambda_{f}(\textbf{v}^{1},...,\textbf{v}^{k}) = \mathbb{E}\left[\Psi\Big(\beta\big(f(y_{1,1}),...,f(y_{1,k})\big),...,\beta\big(f(y_{m,1}),...,f(y_{m,k})\big)\Big)\right],
\end{equation}
where: $\textbf{v}^{1},...,\textbf{v}^{k} \in \mathbb{R}^{n}$, $y_{i,j} = \langle \textbf{r}^{i},\textbf{v}^{j} \rangle$ for $i \in \{1,...,m\}$, $j \in \{1,...,k\}$, expectation is taken over independent random choices $\textbf{r}^{1},...,\textbf{r}^{m}$ from $n$-dimensional Gaussian distributions where each entry is independently taken from $\mathcal{N}(0,1)$, $\langle \cdot \rangle$ denotes the dot product
and $f: \mathbb{R} \rightarrow \mathbb{R}$, $\beta: \mathbb{R}^{k} \rightarrow \mathbb{R}$, $\Psi: \mathbb{R}^{m} \rightarrow \mathbb{R}$ for some integers $k,m > 0$.
We will assume that $\textbf{v}^{1},...,\textbf{v}^{k}$ are linearly independent.
$\Lambda_f$ is always spherically-invariant.

Below we present several examples of machine learning distances/similarity measures that can be expressed in the form: $\Lambda_{f}(\textbf{v}^{1},...,\textbf{v}^{k})$. Although our theoretical results will cover more general cases, the examples will focus on when $k=2$, $\Psi(x_1,\ldots,x_m)=\frac{x_1+\ldots +x_m}{m}$ and $\beta(x,y)=x\cdot y$. Equation~(\ref{main_function}) simplifies:
\begin{equation}
\label{main_function_simple}
\Lambda_{f}(\textbf{v}^{1},\textbf{v}^{2}) = \mathbb{E}\left[f(\langle \textbf{r}, \textbf{v}^{1}\rangle )\cdot f(\langle \textbf{r}, \textbf{v}^{2}\rangle)\right].
\end{equation}
This defines a wide class of spherically invariant kernels characterized by $f$. Our results will cover general functions $f$ that do not have to be linear, or even not continuous. \newline

\textbf{1. Euclidean inner product}

This is probably the most basic example. Let $f(x)=x$. 
One can easily note that $\Lambda_{f}(\textbf{v}^{1},\textbf{v}^{2}) = \langle \textbf{v}^{1}, \textbf{v}^{2}\rangle$. Furthermore, if we take $m$ large enough then the value that will be computed,
$\Psi(\beta(\langle \textbf{r}^{1},\textbf{v}^{1}\rangle, \langle \textbf{r}^{1},\textbf{v}^{2} \rangle),...,\beta(\langle \textbf{r}^{m},\textbf{v}^{1}\rangle, \langle \textbf{r}^{m},\textbf{v}^{2} \rangle))$, is well concentrated around its mean and that follows from standard concentration inequalities.
Since $m$ is usually much smaller than $n$, then one can think about the mapping $\textbf{v} \rightarrow 
(\langle \textbf{r}^{1},\textbf{v} \rangle,...,\langle \textbf{r}^{m},\textbf{v} \rangle)$ as a dimensionality reduction procedure that preserves Euclidean inner products. And indeed, the above transformation is well known in the literature as the aforementioned \textit{Johnson-Lindenstrauss transform}. \newline

\textbf{2. Angular distance}
\label{example:circ}

Now we want to express the angular distance $\theta_{\textbf{v}^{1},\textbf{v}^{2}}$ between two given 
vectors $\textbf{v}^{1},\textbf{v}^{2}$ (that are of not necessarily of the same magnitude) as a 
function $\Lambda_{f}(\textbf{v}_{1},\textbf{v}_{2})$. We take $f$ as the heaviside step function, 
i.e. $f(x) = 1$ for $x \geq 0$ and $f(x) = 0$ otherwise. From basic properties of the Gaussian distribution \cite{choromanska2015binary} 
one can deduce that $\Lambda_{f}(\textbf{v}^{1},\textbf{v}^{2}) = \frac{\theta_{\textbf{v}^{1},\textbf{v}^{2}}}{2\pi}$. Note that, since 
in this setting $f$ takes values from a discrete set, the mapping 
$\textbf{v} \rightarrow (f(\langle \textbf{r}^{1},\textbf{v} \rangle),...,f(\langle \textbf{r}^{m},\textbf{v} \rangle))$ is not only a 
dimensionality reduction, but in fact a hashing procedure encoding angular distance between vectors in terms of the dot product between 
corresponding hashes taken from $\{0,1\}^{m}$.  \newline 

\textbf{3. Arc-cosine and Gaussian kernels}

The arc-cosine kernel \cite{cho2009kernel} is parametrized by $b=0,1,\ldots $, with $f(x) = x^b$ for $x \geq 0$ and $f(x) = 0$ otherwise. For $b=0$ its computation reduces to the computation of the angular distance. If $b=1$ then $f$ is the linear rectifier. Higher-order arc-cosine kernels can be obtained by recursively applying that transformation and thus can be approximated by recursively applying the presented mechanism.
Gaussian kernels can be computed by a similar transformation, with $f$ replaced by trigonometric functions: $\sin(x)$ and $\cos(x)$.
\newline

Our goal is to compute $\Lambda_{f}(\textbf{v}^{1},...,\textbf{v}^{k})$ efficiently. Since the random variable in equation~(\ref{main_function}),
$\Psi(\beta(f(y_{1,1}),...,f(y_{1,k})),...,\beta(f(y_{m,1}),...,f(y_{m,k})))$, is usually well 
concentrated around its mean $\Lambda_{f}(\textbf{v}^{1},...,\textbf{v}^{k})$, 
its straightforward computation gives a good quality approximation of $\Lambda_{f}(\textbf{v}^{1},...,\textbf{v}^{k})$. 
This, as already mentioned, unfortunately usually requires $\Omega(mn)$ time and $\Omega(mn)$ space. 

We are interested in providing a good quality approximation of $\Lambda_{f}(\textbf{v}^{1},...,\textbf{v}^{k})$ 
in subquadratic time and subquadratic (or even linear) space. To achieve this goal, we will replace 
the sequence of independent Gaussian vectors $\textbf{r}^{1},...,\textbf{r}^{m}$ by Gaussian vectors $\textbf{a}^{1},...,\textbf{a}^{m}$ that are no 
longer independent, yet provide us speed-ups in computations and reduce storage complexity. 
Our mechanism will use $t$ independent Gaussian variables $g_{i}$ to construct structured matrices $\textbf{A}=\{\textbf{a}^{i}:i=1,...,m\}$, 
where $\textbf{a}^{i}$ stands for the $i^{th}$ row. Parameter $t$ enables us to make a smooth transition from the unstructured setting 
(large values of $t$), where we obtain stronger concentration results regarding $\Lambda_{f}(\textbf{v}^{1},...,\textbf{v}^{k})$
but need more space and computational time, to the structured setting, where concentration results are weaker (yet still strong 
enough so that the entire mechanism can be applied in practice) but computation can be substantially sped up and the storage complexity is much smaller.

The core of our structured mechanism is the construction of our structured matrices. 
In the next subsection we will present it and show why several structured matrices considered so far in that context are very special cases of our general structured approach.
\subsection{Structured linear projections}
\label{sec:structured}
Consider a vector of independent Gaussian variables $\textbf{g} = (g_{0},...,g_{t-1})$ taken from $\mathcal{N}(0,1)$.
Let $\mathcal{P} = (\textbf{P}_{1},...,\textbf{P}_{m})$ be a sequence of matrices, where: $\textbf{P}_{i} \in \mathbb{R}^{t \times n}$. 
We construct the rows of our structured matrix $\textbf{A}$ as follows:
\begin{equation}
\textbf{a}^{i} = \textbf{g} \cdot \textbf{P}_{i}
\end{equation}
for $i=1,...,m$.
Thus the entire structured mechanism is defined by parameter $t$, and a sequence $\mathcal{P}$. We call it a \textit{$\mathcal{P}$-model}.

In practice we will not store the entire sequence but just a matrix $\textbf{A}$ that is obtained by applying matrices of $\mathcal{P}$ to the vector $\textbf{g}$.
We denote the $r^{th}$ column of matrix $\textbf{P}_{i}$ as $\textbf{p}^{i}_{r}$.
We will assume that sequence $\mathcal{P}$ is \textit{normalized}.
\begin{definition} (Normalization property)
A sequence of matrices $\mathcal{P}=(\textbf{P}_{1},...,\textbf{P}_{m})$ is normalized if for any fixed $i$ and $r$ 
the expression $\|\textbf{p}^{i}_{r}\|_{2}=1$. 
\end{definition}

Note that from the normalization property it follows that every $\textbf{a}^{i}$ is a Gaussian vector with elements 
from $\mathcal{N}(0,1)$. 
We will use another useful notation, namely: $\sigma_{i_{1},i_{2}}(n_{1},n_{2}) = \langle \textbf{p}^{i_1}_{n_{1}},\textbf{p}^{i_2}_{n_{2}}\rangle$ 
for $1 \leq n_{1},n_{2} \leq n$ and $1 \leq i_{1},i_{2} \leq m$. 
Note that when $i=i_{1}=i_{2}$ then $\sigma_{i,i}(n_{1},n_{2})$ reduces to the cross-correlation between 
$n_{1}^{th}$ column and $n_{2}^{th}$ column of $\textbf{P}_{i}$.
If the following is also true: $n_{1}=n_{2}$ then $\sigma_{i,i}(n_{1},n_{2}) = 1$.

We define now graphs associated with a given $\mathcal{P}$-model that we call the \textit{coherence graphs}.

\begin{definition}(Coherence graphs)
Let $1 \leq i_{1},i_{2} \leq m$. We define by $\mathcal{G}_{i_{1},i_{2}}$ an undirected graph
with the set of vertices $V(\mathcal{G}_{i_{1},i_{2}}) = \{\{n_{1},n_{2}\} : 1 \leq n_{1} < n_{2} \leq n$ and $\sigma_{i_{1},i_{2}}(n_{1},n_{2}) \neq 0\}$ and the 
set of edges $E(\mathcal{G}_{i_{1},i_{2}}) = \{\{\{n_{1},n_{2}\},\{n_{2},n_{3}\}\}:\{n_{1},n_{2}\},\{n_{2},n_{3}\} \in V(\mathcal{G}_{i_{1},i_{2}})\}$. 
In other words, edges are between these vertices for which the corresponding $2$-element subsets intersect.
\end{definition}

We denote by $\chi(i_{1},i_{2})$ the \textit{chromatic number of the graph $\mathcal{G}_{i_{1},i_{2}}$}, i.e. the minimum number of 
colors that need to be used to color its vertices in such a way that no two adjacent vertices get the same color.

The correlation between different rows of the structured matrix $\textbf{A}$ obtained from the sequence of 
matrices $\mathcal{P}$ and the ``budget of randomness'' $(g_{0},...,g_{t-1})$ can be measured very accurately by three quantities that 
we will introduce right now. These quantities give a quantitative measure of the ``structuredness'' of a given matrix $\textbf{A}$ and play important role in 
establishing theoretical results for general structured models.

\begin{definition}(Chromatic number of a $\mathcal{P}$-model)
The chromatic number $\chi[\mathcal{P}]$ of a $\mathcal{P}$-model is defined as:
\begin{equation}
\chi[\mathcal{P}] = \max_{1 \leq i,j \leq m} \chi(i,j).
\end{equation}
\end{definition}

Thus the chromatic number of a $\mathcal{P}$-model is the maximum chromatic number of a coherence graph.

\begin{definition}(Coherence and unicoherence of a $\mathcal{P}$-model)
The coherence of a $\mathcal{P}$-model is defined as:
\begin{equation}
\mu[\mathcal{P}] = \max_{1 \leq i,j \leq m} \sqrt{\frac{\sum_{1 \leq n_{1} < n_{2} \leq n} \sigma_{i,j}^{2}(n_{1},n_{2})}{n}}.
\end{equation}
The unicoherence of a $\mathcal{P}$-model is given by the following formula:
\begin{equation}
\tilde{\mu}[\mathcal{P}] = \max_{1 \leq i<j \leq m} \sum_{n_{1}=1}^{n}
|\sigma_{i,j}(n_{1},n_{1})|.
\end{equation}
\end{definition}

We will show in the theoretical section that as long as $\chi[\mathcal{P}],\mu[\mathcal{P}]$ are at most polynomial in $n$ and $\tilde{\mu}[\mathcal{P}] = o(\frac{n}{\log^{2}(n)})$,
strong concentration results regarding the quality of the structured embedding can be derived.
Below we show many classes of matrices $\textbf{A}$ that can be constructed according to the presented mechanism and for which all three quantities have desired orders of magnitude. \newline

\textbf{1. Circulant matrices}

This is a flagship example of the structured approach \cite{hinrichs2011johnson, krahmer2014suprema, rauhut2009circulant, rauhut2012restricted, romberg2009compressive, vybiral2011variant, zhang2013new}. In that setting $t=n$ and the structured Gaussian matrix $\textbf{A}$ is obtained from a single 
Gaussian vector $(g_{0},...,g_{n-1}) \in \mathbb{R}^{n}$ by its right shifts, i.e.
$\textbf{A}$ is of the form:
\begin{equation}
\textbf{A}_{circ} = 
\left( \begin{array}{cccc}
g_{0}  &  g_{1} & ... & g_{n-1}\\
g_{n-1}  &  g_{0} & ... & g_{n-2}\\
...  &  ... & ... & ...\\
g_{n-m+1}  &  g_{n-m+2} & ... & g_{2n-m}\\
\end{array} \right),
\end{equation}

where the operations on indices are taken modulo $n$.
Matrix $\textbf{A}_{circ}$ can be obtained from the presented pipeline by using budget of randomness $(g_{0},...,g_{n-1})$ and a sequence 
of matrices $\mathcal{P}=(\textbf{P}_{1},...,\textbf{P}_{m})$, where $\textbf{P}_{1}, \textbf{P}_{2}, \textbf{P}_{3}, ...$ are respectively: 
\begin{equation*}
\resizebox{0.96\linewidth}{!}{%
 $\left( \begin{array}{cccccc}
\fbox{\textbf{1}} & 0 & ... & ... & ... & 0\\
0 & \fbox{\textbf{1}} & ... & ... & ... & 0 \\
... & ... & ... & ... & ... & ...\\
... & ... & ... & ... & ... & ...\\
0 & ... & ... & ...& ...& \fbox{\textbf{1}} \\
\end{array} \right),
 \left( \begin{array}{cccccc}
0 & \fbox{\textbf{1}} & ... & ... & ...& 0\\
0 & 0 & \fbox{\textbf{1}} & ... & ... & 0 \\
... & ... & ... & ... & ... & ...\\
... & ... & ... & ... & ... & ...\\
\fbox{\textbf{1}} & ... & ... & ... &  ... & 0 \\
\end{array} \right),
\left( \begin{array}{cccccc}
0 & 0 & \fbox{\textbf{1}} & ... & ... & 0\\
0 & 0 & 0 & \fbox{\textbf{1}} & ... & 0 \\
... & ... & ... & ... & ... & ...\\
\fbox{\textbf{1}} & ... & ... & ... &  ... & 0 \\
0 & \fbox{\textbf{1}} & ... & ... &  ... & 0 \\
\end{array} \right),...$
}
\end{equation*}

Each $\mathcal{P}$ is normalized. Furthermore:
\begin{equation}
\sigma_{i_{1},i_{2}}(n_{1},n_{2}) =
\left\{
	\begin{array}{ll}
		0  & \mbox{if }  n_{1} - n_{2} \neq i_{1} - i_{2}  \mod n, \\
		1 & \mbox{otherwise.} 
	\end{array}
\right.
\end{equation}
The above observation implies that each $\mathcal{G}_{i_{1},i_{2}}$ is a collection of vertex disjoint cycles (since each vertex has degree two),
thus in particular $\chi(i_{1},i_{2})$ is at most $3$ (see Figure \ref{fig:circulant} for an illustration). We conclude that $\chi[\mathcal{P}] \leq 3$.
One can also see that $\mu[\mathcal{P}] = O(1)$ and $\tilde{\mu}[\mathcal{P}] = 0$. \newline

\textbf{2. Toeplitz matrices}

Toeplitz matrices that are also used frequently in the structured setting can be modeled by our mechanism by increasing the budget of randomness from $t=n$ to $t=n+m-1$.
A Toeplitz Gaussian matrix is of the form:
\begin{equation}
\textbf{A}_{Toeplitz} = 
\left( \begin{array}{ccccccc}
g_{0}  &  g_{1} & g_{2} & g_{3} & ... & ... & g_{n-1}\\
g_{n}  &  g_{0} & g_{1} & g_{2} & ... & ... & g_{n-2}\\
g_{n+1}  &  g_{n} & g_{0} & g_{1} & ... & ... & g_{n-3}\\
g_{n+2}  &  g_{n+1} & g_{n} & g_{0} & ... & ... & g_{n-4}\\
... &  ... & ... & ... & ... & ... & ...\\
g_{n+m-2}  &  g_{n+m-3} & ... & ... & ... & ... & g_{n-m}\\
\end{array} \right)
\end{equation}

In other words, a Toeplitz matrix is constant along each diagonal.
In that scenario matrices $\textbf{P}_{1},...,\textbf{P}_{m}$ are of the form:
\begin{equation*}
\resizebox{0.96\linewidth}{!}{%
$ \left( \begin{array}{ccccccc}
\fbox{\textbf{1}} & 0 & ... & ... & ... & ... & 0\\
0 & \fbox{\textbf{1}} & ... & ... & ... & ... & 0 \\
... & ... & ... & ... & ... & ... & ...\\
0 & 0 & 0 & 0 & ... & ... & \fbox{\textbf{1}}\\
0 & 0 & 0 & 0 & ... & ... & 0\\
... & ... & ... & ...& ...& ... & ...\\
... & ... & ... & ...& ...& ... & ...\\
0 & 0 & 0 & 0 & ... & ... & 0\\
0 & 0 & 0 & 0 & ... & ... & 0\\
0 & 0 & 0 & 0 & ... & ... & 0\\
\end{array} \right),
 \left( \begin{array}{ccccccc}
0 & \fbox{\textbf{1}} & ... & ... & ...& ... & 0\\
0 & 0 & \fbox{\textbf{1}} & ... & ... & ... & 0\\
... & ... & ... & ... & ... & ... & ...\\
0 & 0 & 0 & 0 & ... & ... & \fbox{\textbf{1}}\\
0 & 0 & 0 & 0 & ... & ... & 0\\
\fbox{\textbf{1}} & ... & ... & ... &  ... & ...& 0 \\
... & ... & ... & ... & ... & ... & ...\\
0 & 0 & 0 & 0 & ... & ... & 0\\
0 & 0 & 0 & 0 & ... & ... & 0\\
0 & 0 & 0 & 0 & ... & ... & 0\\
\end{array} \right),
\left( \begin{array}{ccccccc}
0 & 0 & \fbox{\textbf{1}} & ... & ... & ... & 0\\
0 & 0 & 0 & \fbox{\textbf{1}} & ... & ... & 0 \\
... & ... & ... & ... & ... & ... & ...\\
0 & ... & ... & ... &  ... & ... & \fbox{\textbf{1}} \\
0 & 0 & 0 & 0 & ...& ... & 0\\
0 & 0 & 0 & 0 & ... & ... & 0\\
0 & \fbox{\textbf{1}} & ... & ... &  ... & ... & 0 \\
\fbox{\textbf{1}} & 0 & ... & ...& ...& ... & 0\\
... & ... & ... & ... &  ... & ... & ... \\
0 & 0 & 0 & 0 & ... & ... & 0\\
\end{array} \right),$%
}
\end{equation*}

and so on. Again, one can easily note that $\mathcal{P}$ is normalized.
For Toeplitz matrices we have:
\begin{equation}
\sigma_{i_{1},i_{2}}(n_{1},n_{2}) =
\left\{
	\begin{array}{ll}
		0  & \mbox{if }  n_{1} - n_{2} \neq i_{1} - i_{2}  \mod n, \\
		1-c(i_{1},i_{2},n_{1},n_{2}) & \mbox{otherwise}, 
	\end{array}
\right.
\end{equation}

for some $c_{i_{1},i_{2},n_{1},n_{2}} \in \{0,1\}$.
By increasing the budget of randomness we managed to decrease $|\sigma_{i_{1},i_{2}}(n_{1},n_{2})|$ 
and that implies better concentration results. Bounds for $\chi[\mathcal{P}]$, $\mu[\mathcal{P}]$, $\tilde{\mu}[\mathcal{P}]$ 
from the circulant setting are valid also here. Figures \ref{fig:circulant} and \ref{fig:toeplitz} demonstrate how increasing the ``budget of randomness'' decreases
the chromatic numbers of the corresponding coherence graphs and thus also parameter $\chi[\mathcal{P}]$.
There we compare the circulant structured approach with the Toeplitz structured approach, where the ``budget of randomness''
is increased. \newline

\textbf{3. Hankel matrices}

These can be obtained in the analogous way as Toeplitz matrices since each Hankel matrix is defined as
the one in which each ascending skew-diagonal from left to right is constant. Thus it is a reflected image 
of the Toeplitz matrix and in particular shares with it all structural properties considered above. \newline

\textbf{4. Matrices with low displacement rank}
\label{sec:displacement}

Several classes of structured matrices can be described by the low value of the parameter called \textit{displacement rank} \cite{kailath1979displacement, kailath1995displacement, sindhwani2015structured}.
In particular, classes of matrices described by the formula: 
\begin{equation}
\textbf{A}_{ldr} = \sum_{i=1}^{r} \textbf{Z}_{1}(\textbf{g}^{i})\textbf{Z}_{-1}(\textbf{h}^{i}),
\end{equation}
where: $\textbf{g}^{1},...,\textbf{g}^{r}, \textbf{h}^{1},...,\textbf{h}^{r}$ are given $n$-dimensional vectors ($\textbf{g}^{i}=(g^{i}_{1},...,g^{i}_{n})$)
and matrices $\textbf{Z}_{1}, \textbf{Z}_{-1}$ are the circulant-shift and skew-circulant-shift matrix respectively (see \cite{sindhwani2015structured} for definitions).
Matrices $\textbf{A}_{ldr}$ have displacement rank $r$ and cover
such families as: circulant and skew-circulant matrices, Toeplitz matrices, inverses of Toeplitz matrices (for $r \geq 2$), products of the form $\textbf{A}_{1}...\textbf{A}_{s}$ for $r \geq 2s$ and 
all linear combinations of the form $\sum_{i=1}^{p} \beta_{i} \textbf{A}^{(i)}_{1}...\textbf{A}^{(i)}_{s}$, 
where $r \geq 2sp$ and $\textbf{A}_{j}, \textbf{A}^{(i)}_{j}$ are Toeplitz matrices or inverses of Toeplitz matrices \cite{sindhwani2015structured}.

Assume now that $\textbf{g}^{1},...,\textbf{g}^{r}$ are independent Gaussian vectors. Note that then $\textbf{A}_{ldr}$ is a special instance of the $\mathcal{P}$-model, 
where the budget of randomness is of the form:
$\textbf{g}=(g^{1}_{1},...,g^{1}_{n},g^{2}_{1},...,g^{2}_{n},...,g^{r}_{1},...,g^{r}_{n})$, $\textbf{P}_{1} \in \mathbb{R}^{nr \times n}$ is 
obtained by vertically stacking matrices 
$\textbf{Z}_{-1}(\textbf{h}^{i})$ for $i=1,...,r$ and $\textbf{P}_{i}$ is obtained from $\textbf{P}_{i-1}$ for $i=2,...$ by vertical circulant shifts applied block-wise.

\begin{figure}
\vspace{-0.12in}
\vspace{-0.1in}
\centering
\includegraphics[width = 4.2in]{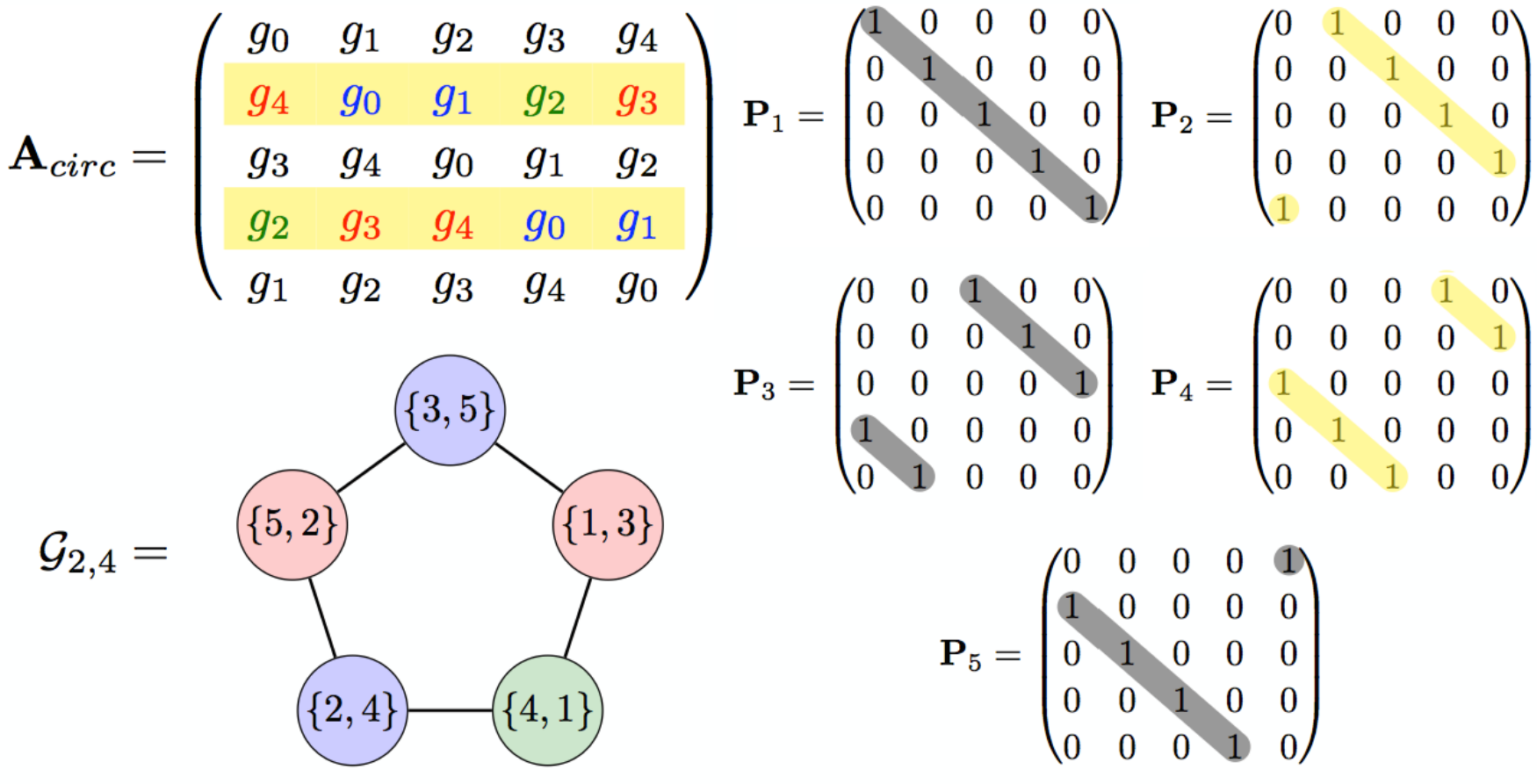}
\vspace{-0.1in}
\caption{Circulant Gaussian matrix with two highlighted rows, corresponding matrices $\textbf{P}_{i}$ and a coherence graph. 
The highlighted rows were obtained by multiplying the ``budget of randomness'' vector $\textbf{g}$
by two highlighted matrices $\textbf{P}_{i}$. The corresponding coherence graph is colored in red, blue and green, matching the coloring of entries in the highlighted rows of $\textbf{A}_{circ}$. The graph is a cycle of length $5$. Since it has an odd number of vertices it requires the use of just three colors for no two adjacent vertices to have the same color, and so its chromatic number is $3$. \\}
\label{fig:circulant}
\vspace{-0.32in}
\end{figure}

There exist several purely deterministic and simple random constructions of the sequence $\textbf{h}^{1},...,\textbf{h}^{r}$ for which the considered parameters $\chi[\mathcal{P}]$, $\mu[\mathcal{P}]$, $\tilde{\mu}[\mathcal{P}]$ of the related $\mathcal{P}$-model are in the
desired range of magnitude. 
For instance, fix some constant $a>0$ and choose at random $a$ nonzero dimensions,
independently for each $\textbf{h}^{i}$. Choose the value of each nonzero dimension to be
$+\frac{1}{\sqrt{ar}}$ with probability $\frac{1}{2}$ and $-\frac{1}{\sqrt{ar}}$ otherwise, independently for each dimension.

In that setting each column of each $\textbf{P}_{i}$ has $L_{2}$ norm equal to one.
One can also note that $\chi[\mathcal{P}] = O(1)$, $\mu[\mathcal{P}]=O(1)$. Furthermore, $\tilde{\mu}[\mathcal{P}]=o(\frac{n}{\log^{2}(n)})$ with high probability if $r$ is large enough (but still satisfies $r=o(n)$).
Thus these matrices can be also used in the algorithm we are about to present now and are covered by our theoretical results. 
It was heuristically observed before that displacement rank $r$ is a useful parameter for tuning the level of ``structuredness'' of these matrices and increasing $r$ may potentially lead to better quality embeddings \cite{sindhwani2015structured}. 
Our framework explains it. Larger values of $r$ trivially imply larger ``budgets of randomness'', i.e. stronger 
concentration results for $\sigma_{i_{1},i_{2}}(n_{1},n_{2})$ and thus much smaller values of $|\sigma_{i_{1},i_{2}}(n_{1},n_{2})|$ in practice. 
That decreases the value of the coherence $\mu[\mathcal{P}]$ and unicoherence $\tilde{\mu}[\mathcal{P}]$ of the related $\mathcal{P}$-model and thus, as we will see in the theoretical section, improves concentration results.

\begin{figure}
\vspace{-0.12in}
\vspace{-0.1in}
\centering
\includegraphics[width = 4.3in]{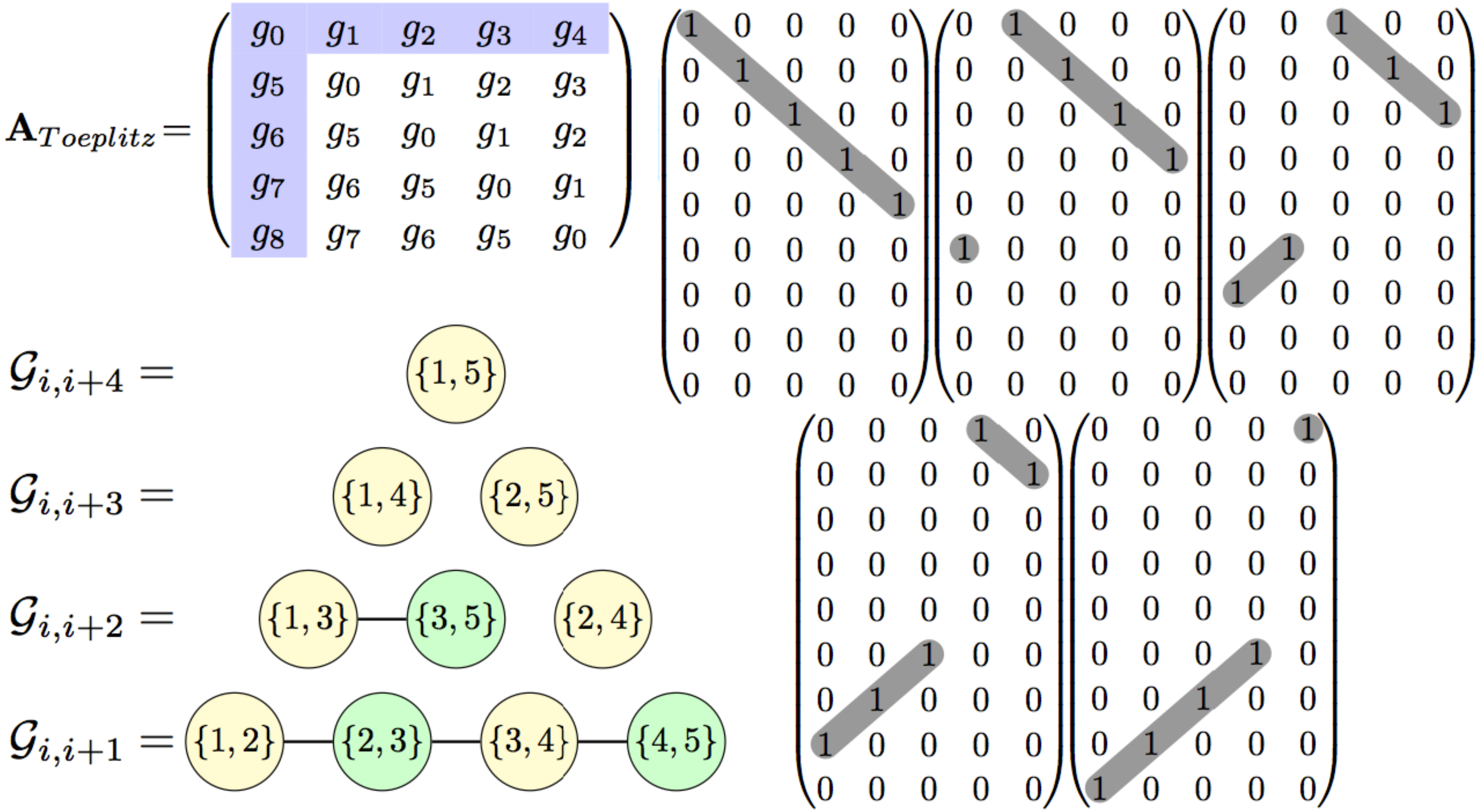}
\vspace{-0.1in}
\caption{Toeplitz Gaussian matrix. This time highlighted row and column correspond to the ``budget of randomness'' vector $\textbf{g}$. On the right: matrices $\textbf{P}_{i}$
of the corresponding $\mathcal{P}$-model. Each row below the Toeplitz Gaussian matrix represents a coherence graph. Any coherence graph of the corresponding
$\mathcal{P}$-model is isomorphic to one of the four presented graphs. Note that for any of these graphs it suffices to use two colors to color them. Thus $\chi[\mathcal{P}] = 2$. Note that a larger ``budget of randomness'' for Toeplitz matrices implies smaller $\chi[\mathcal{P}]$ than in the circulant setting (for the graph from Figure \ref{fig:circulant} we needed three colors)
and that, as we will see soon, will imply better concentration results. \\}
\label{fig:toeplitz}
\vspace{-0.32in}
\end{figure}

\subsection{The Algorithm}
\label{sec:algorithm}

We are ready to describe a general algorithm for fast nonlinear embeddings via structured matrices.
Consider a function
\begin{equation}
\Lambda_{f}(\textbf{v}^{1},...,\textbf{v}^{k}) = \mathbb{E}[\Psi(\beta(f(y_{1,1}),...,f(y_{1,k})),...,\beta(f(y_{m,1}),...,f(y_{m,k})))],
\end{equation}
where: $y_{i,j} = \langle \textbf{r}^{i},\textbf{v}^{j} \rangle$ and $\textbf{r}^{i}$s are independent Gaussian vectors. 
We want to compute $\Lambda_f$ efficiently for a dataset $\mathcal{X}$ of $n$-dimensional points.

Structured matrix $\textbf{A}$ that allows us to do this is constructed by choosing the budget of randomness 
$\textbf{g}=(g_{1},...,g_{t-1})$ for a given parameter $t>0$ and a sequence of matrices 
$\mathcal{P}=\{\textbf{P}_{1},...,\textbf{P}_{m}\}$ such each element of $\textbf{g} \cdot \textbf{P}_{i}$ has the same distribution as the corresponding element of $\textbf{r}^{i}$ for $i=1,...,m$.
By choosing different $\textbf{P}_{i}$s and budgets of randomness $\textbf{g}$ one can smoothly balance between speed of the transform/storage complexity and its quality.

\textbf{Step 1:} Dataset $\mathcal{X}$ is first preprocessed by multiplying each datapoint by a matrix $\textbf{D}_{1}\textbf{HD}_{0}$, where $\textbf{H}$ is an arbitrary 
$L_{2}$-normalized Hadamard matrix 
and $\textbf{D}_{0},\textbf{D}_{1}$ are independent random diagonal matrices with nonzero entries taken from the set $\{-1,+1\}$, each independently at random and with probability $\frac{1}{2}$.

\textbf{Step 2:} Dataset $\mathcal{X}^{\prime}$ is transformed by multiplying it by  a structured matrix $\textbf{A}$.
Then function $f$ is applied pointwise to each datapoint of $\textbf{A}\mathcal{\hat{X}}$. For any given $\textbf{v}^{1},...,\textbf{v}^{k} \in \mathbb{R}^{n}$ the approximated
value of $\Lambda_{f}(\textbf{v}^{1},...,\textbf{v}^{k})$ is calculated as:
\begin{equation}
\label{structured}
\Psi(\beta(\textbf{v}^{f,1}_{1},...,\textbf{v}^{f,k}_{1}),...,\beta(\textbf{v}^{f,1}_{m},...,\textbf{v}^{f,k}_{m})),
\end{equation} 
where $\textbf{v}^{f,i}=f(\textbf{AD}_{1}\textbf{HD}_{0}\textbf{v}^{i})$ with $f$ applied pointwise ($\textbf{v}^{f,i}_{j}$ is the $j^{th}$ dimension of $\textbf{v}^{f,i}$).

In practice, for $k=2$ equation~(\ref{structured}) very often boils down to computing the standard dot product
between $\textbf{v}^{f,1}$ and $\textbf{v}^{f,2}$ (as is the case for any $\Lambda_{f}$ in the form of equation~(\ref{main_function_simple})).

By using structured matrices listed in Section \ref{sec:structured} one can significantly reduce storage complexity of the entire computational mechanism. Indeed:
\begin{remark}
 Circulant, Toeplitz, Hankel matrices or products/linear combinations of the $O(1)$ number of Toeplitz matrices/inverses of Toeplitz matrices can 
 be stored in linear space. Hadamard matrices can be efficiently (i.e. in the subquadratic time) computed on-the-fly and do not have to be stored.
\end{remark} 
More importantly, the presented structured pipeline gives significant computational speed-ups
over the standard approach requiring quadratic time. This is a direct implication of the fact that matrix-vector multiplication, which is a main computational bottleneck of the nonlinear embeddings pipelines, can be performed in subquadratic time for many classes of the structured matrices covered by the presented scheme, in particular for all special classes listed by us so far. Indeed:
 
\begin{remark}
For classes of matrices with bounded displacement rank matrix-vector multiplication can be performed in subquadratic time. These classes cover in particular: circulant and skew-circulant matrices, Toeplitz matrices, Hankel matrices, inverses of Toeplitz matrices (for $r \geq 2$), products of the form $\textbf{A}_{1}...\textbf{A}_{s}$ for $r \geq 2s$ 
and all linear combinations of the form $\sum_{i=1}^{p} \beta_{i} \textbf{A}^{(i)}_{1}...\textbf{A}^{(i)}_{s}$, where $r \geq 2sp$ and $\textbf{A}_{j}, \textbf{A}^{(i)}_{j}$ are Toeplitz matrices or inverses of Toeplitz matrices. For $m \times n$ Toeplitz (and thus also circulant) matrices as well as for Hankel matrices the computation can be done in $O(n\log(m))$ time.
\end{remark} 

Some of the mentioned structured matrices were used before in the non-linear embedding setting for certain functions $f$. However to the best of our knowledge, we are the first to present a general structured framework that covers all these settings as very special subcases. Furthermore, we give rigorous theoretical results proving the quality of the structured approach for general nonlinear functions $f$. The nonlinear transformation is what makes the entire theoretical analysis challenging and forces us to apply different techniques than those for the fast Johnson-Lindenstrauss transform.
\section{Theoretical results}

In this section we prove several concentration results regarding the presented structured mechanism.
We start with the following observation.

\begin{lemma}
\label{unbiasedness_lemma}
Assume that $\Psi(x_{1},...,x_{m})$ is a linear function and in a given $\mathcal{P}$-model
for every $\textbf{P}_{i}$ any two columns of $\textbf{P}_{i}$ are orthogonal.
Then that $\mathcal{P}$-model mechanism gives an unbiased estimation of $\Lambda_{f}(\textbf{v}^{1},...,\textbf{v}^{k})$
i.e. for any given $\textbf{v}^{1},...,\textbf{v}^{k} \in \mathbb{R}^{n}$
the following is true:
\begin{equation}
\mathbb{E}[\Lambda^{struct}_{f}(\textbf{v}^{1},...,\textbf{v}^{k})] = \Lambda_{f}(\textbf{v}^{1},...,\textbf{v}^{k}).
\end{equation}
\end{lemma}

We call the condition regarding matrices $\textbf{P}_{i}$ from the statement above the \textit{orthogonality condition}.
The orthogonality condition is trivially satisfied by Hankel, circulant or Toeplitz structured matrices produced by the $\mathcal{P}$-model.
It is also satisfied in expectation (which in practice suffices) for some structured models where matrices $\textbf{P}_{i}$ are constructed
according to a random procedure. Linear $\Psi$ is used in all applications given by us. We want to note however that even if $\Psi$ is not linear,
strong concentration results (with an extra error accounting for $\Psi$'s nonlinearity) can be obtained as we show in the section regarding concentration inequalities.

\begin{proof}
Note that it suffices to show that every row $\textbf{a}^{i}$ of a structured matrix has the same distribution as the corresponding
row $\textbf{r}^{i}$ of the unstructured matrix. If this is the case then for any given 
$\textbf{v}^{1},...,\textbf{v}^{k} \in \mathbb{R}^{n}$ the distribution of $\beta(\langle\textbf{a}^{i},\textbf{v}^{1}\rangle,...,\langle\textbf{a}^{i},\textbf{v}^{k}\rangle)$
is the same as a distribution of $\beta(\langle\textbf{r}^{i},\textbf{v}^{1}\rangle,...,\langle\textbf{r}^{i},\textbf{v}^{k}\rangle)$
and the result follows from the linearity of expectations.
The fact that a distribution of $\textbf{a}^{i}$ is the same as of $\textbf{r}^{i}$ is implied by two observations.
First, notice that by the way $\textbf{a}^{i}s$ are constructed, the distribution of each dimension of $\textbf{a}^{i}$
is the same as a distribution of the corresponding dimension of $\textbf{r}^{i}$. The independence of different dimensions of
$\textbf{a}^{i}$ is an immediate consequence of the fact that projections of the ``budget of randomness'' Gaussian 
vector $\textbf{g}$ onto orthogonal directions are independent and the assumed orthogonality condition regarding the $\mathcal{P}$-model.  
\end{proof}

From now on we will assume that a given $\mathcal{P}$-model satisfies the orthogonality condition.
We need to introduce a few useful definitions.\begin{definition}
We denote by $\Delta^{\tau}_{a}$ the supremum of the expression $\|\tau(y_{1},...,y_{m}) - \tau(y^{\prime}_{1},...,y^{\prime}_{m})\|$ 
over all pairs of vectors $(y_{1},...,y_{m}), (y^{\prime}_{1},...,y^{\prime}_{m})$ from the domain that differ on at most one dimension and by at most $a$.
\end{definition}

In lots of applications (such as angular distance computation or any $\Lambda_{f}$ in the form of equation~(\ref{main_function_simple})) we have: 
$\Psi(y_{1},...,y_{m}) = \frac{y_1+...+y_m}{m}$. In that setting $\Psi$ is $\frac{1}{m}y_{diff}$-bounded for 
$y_{diff} = sup_{y \in \mathcal{D}} y - \inf_{y \in \mathcal{D}} y$. In the angular distance setting we have: 
$y_{diff} = 1$. For $\Psi$ given above we also have: $\Delta^{\Psi}_{a} \leq \frac{a}{m}$.

\begin{definition}
For a function $\Psi(\beta(y_{1,1},...,y_{1,k}),...,\beta(y_{m,1},...,y_{m,k}))$ we denote  
\begin{equation}
\rho_{i}^{\Psi,\beta} = \sup_{y_{1,1},...,y_{m,k},y^{\prime}_{1,1},...,y^{\prime}_{m,k}} |h(y_{1,1},...,y_{m,k}) - h(y^{\prime}_{1,1},...,y^{\prime}_{m,k})|,
\end{equation}
where  $h(x_{1,1},...,x_{m,k}) = \Psi(\beta(x_{1,1},...,x_{1,k}),...,\beta(x_{m,1},...,x_{m,k}))$
and sequences $(y_{1,1},...,y_{m,k})$, $(y^{\prime}_{1,1},...,y^{\prime}_{m,k})$ differ on the $i^{th}$
coordinate. 
\end{definition}

For instance, for the angular distance setting we have: $\rho_{i}^{\Psi, \beta} \leq \frac{1}{m}$.

Note that the value of the main computational block of $\Lambda_{f}(\textbf{v}^{1},...,\textbf{v}^{k})$,
namely: 
\begin{equation}
B^{\textbf{v}^{1},...,\textbf{v}^{k}} = \beta(f(\langle \textbf{r},\textbf{v}^{1} \rangle),...,f(\langle \textbf{r},\textbf{v}^{k} \rangle)),
\end{equation}
depends only on the
projection of $\textbf{r}$ into the linear space spanned by $\textbf{v}^{1},...,\textbf{v}^{k}$, not the part orthogonal to it.
Thus for fixed $\textbf{v}^{1},...,\textbf{v}^{k}$, $f$ and $\beta$ function $B^{\textbf{v}^{1},...,\textbf{v}^{k}}$ is in fact the 
function $B^{\textbf{v}^{1},...,\textbf{v}^{k}}(\textbf{r}_{proj})$ of $\textbf{r}_{proj}$, where the $j^{th}$ coordinate of $\textbf{r}_{proj}$ is the projection of $\textbf{r}$ onto $\textbf{v}^{j}$.
We will measure how sensitive $B$ is to the perturbations of $\textbf{r}_{proj}$ using the following definition:

\begin{definition}
Let $\Lambda_{f} : \mathbb{R}^{k} \rightarrow \mathbb{R}$ be as in equation (\ref{main_function}). Define:
\begin{equation}
p_{\lambda,\epsilon} =
\mathbb{P}[ \sup_{ \textbf{v}^{1},...,\textbf{v}^{k}, \|\zeta\|_{\infty} \leq \epsilon}|B^{\textbf{v}^{1},...,\textbf{v}^{k}}(\textbf{r}_{proj}+\zeta)-
B^{\textbf{v}^{1},...,\textbf{v}^{k}}(\textbf{r}_{proj})| > \lambda],
\end{equation}
where the supremum is taken over all $k$-tuples of linearly independent vectors from the domain.
We also denote 
\begin{equation}
\tilde{\beta}_{\epsilon} = \sup_{ \textbf{v}^{1},...,\textbf{v}^{k}, \|\zeta\|_{\infty} \leq \epsilon} 
|\mathbb{E}[B^{\textbf{v}^{1},...,\textbf{v}^{k}}(\textbf{r}_{proj}+\zeta)]-\mathbb{E}[
B^{\textbf{v}^{1},...,\textbf{v}^{k}}(\textbf{r}_{proj})]|.
\end{equation}
\end{definition}

\textbf{Example ($\beta(x,y) = x \cdot y$)}
If $k=2$, $\beta(x,y) = x \cdot y$ (as it is the case in most of the considered examples)
and data is taken from the bounded domain then one can easily see that $\tilde{\beta}_{\epsilon}=O(\epsilon)$ for $\epsilon < 1$.\newline

\textbf{Example - angular case.}
For the angular distance setting one can prove (see: Appendix) that $p_{0,\epsilon} \leq \frac{2\sqrt{2}m\epsilon}{\pi} + \frac{2}{\pi m^{2}}$. \newline

\textbf{Example - general kernels.}
We say that function $f:\mathbb{R} \rightarrow \mathbb{R}$ is \textit{$(\eta,\rho)$-Lipschitz} if
$|x-y| \leq \eta \implies |f(x)-f(y)| \leq \rho$. 
Let $\Theta = \max_{\textbf{v} \in \mathcal{X}} \|\textbf{v}\|_{2}$ and let $f_{max}$ be the maximum value of the bounded function $|f|$. 
If $\Lambda_{f}$ is in the form of equation (\ref{main_function_simple}) and $f$ is $(\Theta \epsilon k, \rho)$-Lipschitz then one can easily prove that:
$p_{\lambda,\epsilon} = 0$ for $\lambda = 2f_{max} \rho  + \rho^{2}$.
For instance, if $f(x) = \cos(x)$ and all datapoints of $\mathcal{X}$ have $L_{2}$-norm at most $1$
then $p_{\lambda, \epsilon} = 0$ for $\lambda = \epsilon k(2 + \epsilon k)$. \newline

\begin{definition}
The $\textit{Legendre Transform}$ $\mathcal{L}_{X}$ of a random variable $X$ is defined as:
{$\mathcal{L}_{X}(x) = \max_{s \in \mathbb{R}}(sx-\log(\mathbb{E}[e^{sX}]))$}. For a $k$-tuple $\textbf{v}^{1},...,\textbf{v}^{k}$
and given $\zeta$ with $\|\zeta\|_{\infty} \leq \epsilon$ we denote  
\begin{equation}
\mathcal{L}_{\zeta, \textbf{v}^{1},...,\textbf{v}^{k}}(x) = 
\mathcal{L}_{X}(x),
\end{equation}
where: $X = B^{\textbf{v}^{1},...,\textbf{v}^{k}}(\textbf{r}_{proj}+\zeta) - \mathbb{E}[B^{\textbf{v}^{1},...,\textbf{v}^{k}}(\textbf{r}_{proj}+\zeta)]$.
\end{definition}

If a nonlinear mapping $f$ is unbounded we will assume that all datapoints are taken from a bounded set and 
that $\mathcal{L}_{X}(\textbf{r}_{proj}+\zeta)(x) \geq c_{1}|x|^{\alpha}$,  $|\mathcal{L}_{X}^{\prime}(\textbf{r}_{proj}+\zeta)(x)| \leq c_{2}|x|^{\gamma}$ for some constants $\alpha,  c_{1} >0$ and $\gamma, c_{2} \geq 0$.
The latter conditions are trivially satisfied in most of the considered structured computations with unbounded $f$. In particular, if $f$ is an arc-cosine kernel then one can take: $\alpha = 1$,
$\gamma = 0$.
Our main result is stated below.

\begin{theorem}
\label{general_theorem}
Let $\mathcal{X}$ be a dataset of $n$-dimensional points and size $N$.
Let $\Lambda_{f}$ be of the form:
\begin{equation}
\Lambda_{f}(\textbf{v}^{1},...,\textbf{v}^{k}) = \mathbb{E}[\Psi(\beta(f(y_{1,1}),...,f(y_{1,k})),...,\beta(f(y_{m,1}),...,f(y_{m,k})))],
\end{equation}
where: $y_{i,j} = \langle \textbf{r}^{i},\textbf{v}^{j} \rangle$ for $i \in \{1,...,m\}$, $j \in \{1,...,k\}$ and $\textbf{r}^{i}$s are independent Gaussian vectors. 
Assume first that $M=\max_{x \in \mathbb{R}^{k}} |\beta(x)| < \infty$. 
Consider the algorithm presented in Subsection \ref{sec:algorithm} for computing $\Lambda_{f}$. Take a class of structured matrices such that $\tilde{\mu}[\mathcal{P}] = o(\frac{n}{\log^{2}(n)})$.
Then for any $K,\lambda, \epsilon>0$, $0 \leq \bar{m} \leq m$ and $n$ large enough the probability that there exists a $k$-tuple of points from $\mathcal{X}$ and such that the value of $\Lambda_{f}$ computed by the algorithm differs from
the correct one by more than $err = K + \bar{m} \Delta^{\Psi}_{M} + (m - \bar{m})\Delta^{\Psi}_{\lambda}$
is at most:
\begin{equation}
{N \choose k}\left(
2km\chi[\mathcal{P}] e^{-\frac{1}{8\chi^{2}[\mathcal{P}]\mu^{2}[\mathcal{P}]}\frac{n}{\log^{6}(n)}}+
k^{2}m^{2}\chi[\mathcal{P}]e^{-\frac{\epsilon^{2}\sqrt{n}}{8\chi^{2}[\mathcal{P}]\mu^{2}[\mathcal{P}]\log^{4}(n)}} + p_{bad}
\right),
\end{equation}
where: $p_{bad} = 2nke^{-\frac{\log^{2}(n)}{8}} + \sqrt{\frac{2mk}{\pi}}e^{-\frac{mk}{2}} + \sum_{j=\bar{m}+1}^{m} \frac{(p_{\lambda,\epsilon}m)^{j}}{j!}
+ 2e^{-\frac{2K^{2}}{\sum_{i=1}^{mk}(\rho_{i}^{\Psi,\beta})^{2}}}$. 

If $\max_{x \in \mathbb{R}^{k}} |\beta(x)| = \infty$ then for $\Psi(x_{1},...,x_{m}) = \frac{x_1+...+x_m}{m}$  the above holds for $m$ large enough (but independent of $n$), with 
$err= \tilde{\beta}_{\epsilon} + m^{-\frac{1}{2\alpha}}$ and $p_{bad}$ given by:
\begin{equation}
p_{bad} = 2nke^{-\frac{\log^{2}(n)}{8}} + \sqrt{\frac{2mk}{\pi}}e^{-\frac{mk}{2}} + O(e^{-\Omega(\sqrt{m})}).
\end{equation}

\end{theorem}

Note first that as noted before, for all specific examples of structured matrices produced by the $\mathcal{P}$-model that we listed before the values 
of the key parameters $\chi[\mathcal{P}]$, $\mu[\mathcal{P}]$ and $\tilde{\mu}[\mathcal{P}]$ are of order that enables us to apply Theorem \ref{general_theorem} 
and obtain sharp concentration results.
Furthermore, terms in the formula for $p_{bad}$ are either already inversely proportional to superpolynomial functions of $m$ or $n$ or can be easily made so by 
appropriate choice of parameters.
Note also that, as we have already mentioned, for $M = \infty$ we get: $\tilde{\beta}_{\epsilon} = O(\epsilon)$ for $\epsilon < 1$ thus $err = O(\epsilon) + m^{-\frac{1}{2\alpha}}$.
Therefore we obtain strong concentration results regarding all $k$-tuples for datasets of superpolynomial sizes for both: $M < \infty$ and $M = \infty$. We are not aware of any other result like that for 
nonlinear embeddings with general structured matrices.
In Theorem \ref{general_theorem} we grouped together probabilities that do not depend on the structure of the chosen matrix (these in the formula for $p_{bad}$) and these that do.
Finally, note that clearly smaller values of $\chi[\mathcal{P}]$ and $\mu[\mathcal{P}]$ improve
concentration results.

Theorem \ref{general_theorem} implies several other structured results. In particular we have:

\begin{theorem}
\label{cor1}
Let $\mathcal{X}$ be as in Theorem \ref{general_theorem}.
Let $\Lambda_{f}(\textbf{v}^{1},\textbf{v}^{2})$ be an angular distance between $\textbf{v}^{1}$ and $\textbf{v}^{2}$.
Consider the algorithm presented in Subsection \ref{sec:algorithm} for computing $\Lambda_{f}$. 
Assume that the class of structured matrices is taken from one of the the following sets: circulant matrices, skew-circulant matrices, Toeplitz matrices, Hankel matrices. Then for $n$ large enough and any $0 < \tau < 0.5$ the probability that there exists a pair of points from $\mathcal{X}$ such that the value of $\Lambda_{f}$ computed by the algorithm differs from the correct one by more than $m^{-\tau}+\frac{1}{\log(m)}$ is at most:
$O(N^{2}e^{-m^{1-2\tau}})$.
\end{theorem}

Let us take now the family of functions $\Lambda_{f}(\textbf{v}^{1},\textbf{v}^{2}) = \mathbb{E}\left[f(\langle \textbf{r}, \textbf{v}^{1}\rangle )\cdot f(\langle \textbf{r}, \textbf{v}^{2}\rangle)\right]$ 
describing general kernels introduced by us in equation~(\ref{main_function_simple}).
Those cover Gaussian kernels and many more.

The following is another corollary of Theorem \ref{general_theorem}:

\begin{theorem}
\label{cor2}
Let $\mathcal{X}$ be a dataset of $N$ points from the $n$-dimensional ball $\mathcal{B}$ of unit $L_{2}$-norm and 
$\Lambda_{f}(\textbf{v}^{1},\textbf{v}^{2}) = \mathbb{E}\left[f(\langle \textbf{r}, \textbf{v}^{1}\rangle )\cdot f(\langle \textbf{r}, \textbf{v}^{2}\rangle)\right]$. 
Assume that $|f|$ is bounded, $f_{max}$ is the maximum value of $|f|$ and $f$ is $(2\epsilon, \rho)$-Lipschitz.
Consider the algorithm presented in Subsection \ref{sec:algorithm} for computing $\Lambda_{f}$. 
Assume that the class of structured matrices is taken from one of the the following sets: circulant matrices, skew-circulant matrices, Toeplitz matrices, Hankel matrices. Then for $n$ large enough and any $0 < \tau < 0.5$ the probability that there exists a pair of points from $\mathcal{X}$ such that the value of $\Lambda_{f}$ computed by the algorithm differs from the correct one by more than $(m^{-\tau}+2f_{max}\rho + \rho^{2})$ is at most:
$O(N^{2}e^{-m^{1-2\tau}f^{-2}_{max}})$.
\end{theorem}

\section{Conclusions}

We presented a general framework for structured computations of multivariate randomized functions 
based on Gaussian sampling. The presented method gives strong theoretical guarantees and,
to the best of our knowledge, covers as special cases all structured approaches used in that setting before.
It can be applied to speed up computations of many kernels that are based on random feature techniques 
and provides convenient parameter tuning the desired level of ``structuredness'' that other approaches do not have.
The presented structured mechanism provides also a significant reduction in space complexity since all structured matrices
that are used can be stored in the subquadratic space. 

\subparagraph*{Acknowledgements.}

We want to sincerely thank Vikas Sindhwani for his insightful comments, encouragement and constant support.

\bibliography{kernels-icalp-2016}

\newpage

\appendix

\section{Proof of Theorem \ref{general_theorem}}

We start with several auxiliary lemmas and definitions.

\begin{definition}
We say that vector $\textbf{x}=(x_{1},...,x_{n})$ of unit $L_{2}$ norm is $\theta$-balanced if 
$|x_{i}| \leq \frac{\theta}{\sqrt{n}}$ for $i=1,...,n$.
\end{definition}

The following standard concentration inequality will be frequently used by us in the proof.

\begin{lemma}(Azuma's Inequality)
Let $X_{1},...,X_{n}$ be a martingale and assume that $-\alpha_{i} \leq X_{i} \leq \beta_{i}$ for some positive constants $\alpha_{1},...,\alpha_{n}, \beta_{1},...,\beta_{n}$. 
Denote $X = \sum_{i=1}^{n} X_{i}$.
Then the following is true:
\begin{equation}
\mathbb{P}[|X - \mathbb{E}[X]| > a] \leq 2e^{-\frac{a^{2}}{2\sum_{i=1}^{n}(\alpha_{i} + \beta_{i})^{2}}}
\end{equation}
\end{lemma}

We call the first of the algorithm, where each datapoint is linearly transformed by a mapping
$\textbf{HD}_{0}$ the $0^{th}$-phase (note that $0^{th}$-phase is a part of the preprocessing step given in the description of the algorithm). We call the second phase of the algorithm, where
each datapoint already linearly transformed by $\textbf{HD}_{0}$ is linearly transformed by $\textbf{AD}_{1}$ and then nonlinearly transformed by applying pointwise mapping $f$ the $1^{st}$-phase.

Our first observation is that the probability that for every $k$-tuple $\textbf{v}^{1},...,\textbf{v}^{k} \in \mathcal{X}$
and every fixed orthonormal basis $\mathcal{B}$ of the $k$-dimensional linear space spanned by $\textbf{v}^{1},...,\textbf{v}^{k} \in \mathcal{X}$ each vector of the basis is $\log(n)$-balanced is very high.
We state it rigorously below.
We denote by $N$ the size of the dataset $\mathcal{X}$.

\begin{lemma}
\label{balanceness_lemma}
Fix an orthonormal basis $\mathcal{B}(\textbf{v}^{1},...,\textbf{v}^{k})$ for every $k$-tuple of independent vectors $\textbf{v}^{1},...,\textbf{v}^{k}$ from $\mathcal{X}$. The probability of the event $\mathcal{E}_{balanced}$ that after $0^{th}$-phase the vectors of each transformed basis corresponding to linearly transformed $k$-tuple are $\log(n)$-balanced is at least: $\mathbb{P}[\mathcal{E}_{balanced}] \geq 1 - 2kn{N \choose k}e^{-\frac{\log^{2}(n)}{8}}$.
\end{lemma}

\begin{proof}
Fix a $k$-tuple of linearly independent vectors $\textbf{v}^{1},...,\textbf{v}^{k}$
and a fixed basis $\mathcal{B} = \{\textbf{x}^{1},...,\textbf{x}^{k}\}$ of
$span(\textbf{v}^{1},...,\textbf{v}^{k})$. Denote $\textbf{x}^{j} = (x^{j}_{1},...,x^{j}_{n})$.
Denote by $\tilde{\textbf{x}}^{j}$ an image of $\textbf{x}^{j}$ under transformation $\textbf{HD}_{0}$. Note that the $i^{th}$ dimension of $\tilde{\textbf{x}}^{j}$
is given by the formula: $\tilde{x}^{j}_{i} = h_{i,1}x^{j}_{1} + ... + h_{i,n}x^{j,n}$,
where $h_{l,u}$ stands for the $l^{th}$ element of the $u^{th}$ column of the randomized
Hadamard matrix $\textbf{HD}_{0}$.
First we use Azuma's Inequality to find an upper bound on the probability
that $|\tilde{x}^{j}_{i}| > a$, where $a=\frac{\log(n)}{\sqrt{n}}$.
By Azuma's Inequality, we have:
\begin{equation}
\mathbb{P}[|h_{i,1}x^{j}_{1} + ... + h_{i,n}x^{j,n}| \geq a] \leq 2e^{-\frac{\log^{2}(n)}{8}}.
\end{equation}
We use: $\alpha_{i} = \beta_{i} = \frac{1}{\sqrt{n}}$.
Now we take union bound over all $n$ dimensions, all $k$ vectors of basis $\mathcal{B}$ and 
all $k$-tuples of linearly independent vectors and the proof is completed.
\end{proof}

Let us notice that clearly each transformed basis is still a system of orthonormal vectors since $\textbf{HD}_{0}$ is an isometry.

Note that $\textbf{p}^{i}_{r}$ stands for the $r^{th}$ column of the matrix $\textbf{P}_{i}$.
Let us denote by $p^{i}_{r,u}$ the $u^{th}$ element of $\textbf{p}^{i}_{r}$.

Fix some orthonormal basis $\textbf{x}^{1},...,\textbf{x}^{k}$.
Our next lemma describes dot products $\langle\textbf{a}^{i},\textbf{x}^{j}\rangle$, where $\textbf{a}^{i}s$ are rows of a structured matrix $\textbf{AD}_{1}$ in terms of $\textbf{x}^{j}$s,
the elements of matrices $\textbf{P}_{i}$, matrix $\textbf{D}_{1}$ and a Gaussian vector $\textbf{g}$.  

\begin{lemma}
\label{s_lemma}
Let $d_{1},...,d_{n}$ be the diagonal entries of a matrix $\textbf{D}_{1}$.
The dot product $\langle\textbf{a}^{i},\textbf{x}^{j}\rangle$ is of the form $\langle\textbf{g},\textbf{s}^{i,j}\rangle$, where $\textbf{s}^{i,j}=(\textbf{s}^{i,j}_{1},...,\textbf{s}^{i,j}_{t})$ is given by the formula:
\begin{equation}
\textbf{s}^{i,j}_{l} = d_{1}p^{i}_{l,1}x^{j}_{1} + ... + d_{n}p^{i}_{l,n}x^{j,n}.
\end{equation}
Furthermore the following holds:
\begin{equation}
\langle\textbf{s}^{i_{1},j_{1}},\textbf{s}^{i_{2},j_{2}}\rangle
= 2 \sum_{1 \leq n_{1} < n_{2} \leq n} d_{n_{1}}d_{n_{2}}x^{j_{1}}_{n_{1}}x^{j_{2}}_{n_{2}}\sigma_{i_{1},i_{2}}(n_{1},n_{2})
\end{equation}
for $j_{1} \neq j_{2}$ and 
\begin{align}
\begin{split}
\langle\textbf{s}^{i_{1},j_{1}},\textbf{s}^{i_{2},j_{2}}\rangle &= 
\sum_{n_{1}=1}^{n}\sigma_{i_{1},i_{2}}(n_{1},n_{1})x^{j_{1}}_{n_{1}}x^{j_{2}}_{n_{1}}\\
&\quad + 2\sum_{1 \leq n_{1} < n_{2} \leq n} d_{n_{1}}d_{n_{2}}x^{j_{1}}_{n_{1}}x^{j_{2}}_{n_{2}}
\sigma_{i_{1},i_{2}}(n_{1},n_{2})
\end{split}
\end{align}
for $i_{1} \neq i_{2}$.
\end{lemma}

\begin{proof}
The first part of the statement follows straightforwardly from the description of the structured mechanism so we leave it to the Reader. 
Let us derive dot products for vectors: $\textbf{s}^{i_{1},j_{1}}$, $\textbf{s}^{i_{2},j_{2}}$.
Consider first the setting, where $i_{1}=i_{2}$.
We have:
\begin{align}
\begin{split}
\langle\textbf{s}^{i_{1},j_{1}},\textbf{s}^{i_{1},j_{2}}\rangle &= x^{j_{1}}_{1}x^{j_{2}}_{1}\sum_{l=1}^{t}(p^{i_{1}}_{l,1})^{2}+...+x^{j_{1}}_{n}x^{j_{2}}_{n}\sum_{l=1}^{t}(p^{i_{1}}_{l,n})^{2}\\
&\quad+2 \sum_{1 \leq n_{1} < n_{2} \leq n} d_{n_{1}}d_{n_{2}}x^{j_{1}}_{n_{1}}x^{j_{2}}_{n_{2}}(\sum_{i=1}^{t} p^{i_{1}}_{l,n_{1}}p^{i_{2}}_{l,n_{2}})
\end{split}
\end{align}

From the normalization property and the fact that $\textbf{x}^{j_{1}}$ is orthogonal to
$\textbf{x}^{j_{2}}$ it follows that the first term in the RHS of the equation above is $0$.

Therefore we obtain:
\begin{equation}
\langle\textbf{s}^{i_{1},j_{1}}, \textbf{s}^{i_{1},j_{2}}\rangle = 2 \sum_{1 \leq n_{1} < n_{2} \leq n} d_{n_{1}}d_{n_{2}}x^{j_{1}}_{n_{1}}x^{j_{2}}_{n_{2}}\sigma_{i_{1},i_{1}}(n_{1},n_{2}).
\end{equation}

In the special case when for any fixed $\textbf{P}_{i}$ any two different columns of $\textbf{P}_{i}$ are orthogonal, the following holds:
$\sigma_{i_{1},i_{1}}(n_{1},n_{2}) = 0$. Therefore we also have: $\langle\textbf{s}^{i_{1},j_{1}}, \textbf{s}^{i_{1},j_{2}}\rangle = 0$. This special case covers in particular circulant, Toeplitz and Hankel matrices.

Now consider the case when $i_{1} \neq i_{2}$.
By the analysis analogous to the one from the previous setting, we obtain:
\begin{align}
\begin{split}
\langle\textbf{s}^{i_{1},j_{1}}, \textbf{s}^{i_{2},j_{2}}\rangle &= \sigma_{i_{1},i_{2}}(1,1)x^{j_{1}}_{1}x^{j_{2}}_{1}+...+\sigma_{i_{1},i_{2}}(n,n)x^{j_{1}}_{n}x^{j_{2}}_{n}\\
&\quad+ 2 \sum_{1 \leq n_{1} < n_{2} \leq n} d_{n_{1}}d_{n_{2}}x^{j_{1}}_{n_{1}}x^{j_{2}}_{n_{2}}\sigma_{i_{1},i_{2}}(n_{1},n_{2}).
\end{split}
\end{align} 
This time in general we cannot get rid of the first term in the RHS expression. However, this can be done if columns of the same indices in different $P_{i}s$ are orthogonal. This is in fact again the case for circulant, Toeplitz or Hankel matrices.
\end{proof}

Let us assume now that for every $k$-tuple $\textbf{v}^{1},...,\textbf{v}^{k}$ of independent
vectors from a dataset $\mathcal{X}$ we fixed an orthonormal basis $\mathcal{B}(\textbf{v}^{1},...,\textbf{v}^{k})$. Let $\textbf{s}^{i,j}$ be as above, where $\textbf{x}^{j}$
stand for a vector from a basis transformed by a linear mapping $\textbf{HD}_{0}$.
Fix some $\kappa > 0$. We will calculate the probability that for all
$\textbf{s}^{i_{1},j_{1}}$, $\textbf{s}^{i_{2},j_{2}}$, where $i_{1} \neq i_{2}$ or $j_{1} \neq j_{2}$ the absolute value of the dot product of $\textbf{s}^{i_{1},j_{1}}$ and $\textbf{s}^{i_{2},j_{2}}$ is at most $\kappa$ and furthermore $\|\textbf{s}^{i,j}\|^{2}_{2}$ is close
to its expected value for every $1 \leq i \leq m$ and $1 \leq j \leq k$.

\begin{lemma}
\label{large_dot_product_lemma}
The probability of the event $\mathcal{E}_{dot}^{\kappa}$ that for all $i_{1},i_{2},j_{1},j_{2}$ such that $i_{1} \neq i_{2}$ or
$j_{1} \neq j_{2}$ the following holds: 
$|\langle\textbf{s}^{i_{1},j_{1}},\textbf{s}^{i_{2},j_{2}} \rangle| \leq \kappa$
and that furthermore 
$\sqrt{1-\frac{1}{\log(n)}} \leq \|\textbf{s}^{i,j}\|_{2} \leq \sqrt{1 + \frac{1}{\log(n)}}$
for every $1 \leq i \leq m$ and $1 \leq j \leq k$ is at least:
\begin{align}
\mathbb{P}[\mathcal{E}_{dot}^{\kappa}] \geq (1 - 2kn{N \choose k}e^{-\frac{\log^{2}(n)}{8}}) \cdot 
(1-k^{2}m^{2}{N \choose k} \chi[\mathcal{P}]e^{-\frac{(\kappa-\frac{\log^2{n}}{n}\tilde{\mu}[\mathcal{P}])^{2}}{8\chi^{2}[\mathcal{P}]\mu^{2}[\mathcal{P}]}\frac{n}{\log^{4}(n)}}-\Gamma),
\end{align}
where $\Gamma = 2km\chi[\mathcal{P}] e^{-\frac{1}{8\chi^{2}[\mathcal{P}]\mu^{2}[\mathcal{P}]}\frac{n}{\log^{6}(n)}}$.

\end{lemma}

\begin{proof}
We will start with the dot product of the form $\langle\textbf{s}^{i_{1},j_{1}},\textbf{s}^{i_{1},j_{2}}\rangle$. From Lemma \ref{s_lemma} we get:
\begin{align}
\mathbb{P}[|\langle\textbf{s}^{i_{1},j_{1}}, \textbf{s}^{i_{2},j_{2}}\rangle| > \kappa] = 
\mathbb{P}[|\sum_{1 \leq n_{1} < n_{2} \leq n} 2d_{n_{1}}d_{n_{2}}x^{j_{1}}_{n_{1}}x^{j_{2}}_{n_{2}}\sigma_{i_{1},i_{2}}(n_{1},n_{2})| > \kappa].
\end{align}
For the structured matrices lots of the terms in the sum above are equal to $0$ since 
$\sigma_{i_{1},i_{2}}(n_{1},n_{2})$ vanishes for them. Thus let us consider random variables $Y_{n_{1},n_{2}}$ of the form $Y_{n_{1},n_{2}} = 2d_{n_{1}}d_{n_{2}}x^{j_{1}}_{n_{1}}x^{j_{2}}_{n_{2}} \sigma_{i_{1},i_{1}}(n_{1},n_{2})$
for $\{n_{1},n_{2}\}$ such that $n_{1} \neq n_{2}$ and $\sigma_{i_{1},i_{1}}(n_{1},n_{2}) \neq 0$.

From the definition of the chromatic number $\chi(i_{1},i_{1})$ we get that we can partition the set of all random variables $Y_{n_{1},n_{2}}$ into at most $\chi(i_{1},i_{1})$ subsets such that
random variables in each subset are independent. The crucial observation is that when the
number of these subsets is small (i.e. the corresponding chromatic number is small) then 
one can obtain sharp lower bounds on $\mathbb{P}[\mathcal{E}_{dot}^{\kappa}]$.
We show it now.
Let us denote the aforementioned subsets as: $\mathcal{L}_{1},...,\mathcal{L}_{r}$, where
$r \leq \chi(i_{1},i_{1})$. Let us denote an event $\{|\sum_{1 \leq n_{1} < n_{2} \leq n} 2d_{n_{1}}d_{n_{2}}x^{j_{1}}_{n_{1}}x^{j_{2}}n_{2}\sigma_{i_{1},i_{1}}(n_{1},n_{2})| > \kappa | \mathcal{E}_{balanced}\}$ as $\mathcal{F}_{\kappa}^{i_{1},i_{1},j_{1},j_{2}}$, where $x^{j}$s correspond to vectors of the basis already transformed by the linear mapping $\textbf{HD}_{0}$, $d_{i}$s are from the diagonal of a random matrix $\textbf{D}_{1}$ and the conditioning is on the event that all $x^{j}$s are $\log(n)$-balanced.
Note that from the union bound we get $\mathcal{F}_{\kappa}^{i_{1},i_{1},j_{1},j_{2}} \subseteq \mathcal{E}_{1} \cup ... \cup \mathcal{E}_{r}$, where each $\mathcal{E}_{j}$ is defined as follows:
\begin{equation}
\mathcal{E}_{j} = \{|\sum_{Y \in \mathcal{L}_{j}} Y| \geq \frac{\kappa}{\chi(i_{1},i_{1})} | \mathcal{E}_{balanced}\}.
\end{equation}
From the union bound we clearly have:
\begin{equation}
\mathbb{P}[\mathcal{F}_{\kappa}^{i_{1},i_{1},j_{1},j_{2}}] \leq \sum_{i=1}^{r} \mathbb{P}[\mathcal{E}_{i}].
\end{equation}
To bound $\mathbb{P}[\mathcal{E}_{i}]$ we first find the upper bound of its variant $\mathbb{P}[\mathcal{E}^{cond}_{i}]$ conditioned
under the choices of $x^{j}$.

The bound on each $\mathbb{P}[\mathcal{E}^{cond}_{i}]$ can be used by applying Azuma's Inequality (or another standard concentration inequality) since random variables involved are independent.
We get:
\begin{equation}
\mathbb{P}[\mathcal{E}^{cond}_{i}] \leq 2e^{-\frac{\frac{\kappa^{2}}{\chi^2(i_{1},i_{1})}}{2 \sum_{1 \leq n_{1} < 
n_{2} \leq n} (2\sigma_{i_{1},i_{1}}(n_{1},n_{2}))^{2}(x^{j_{1}}_{n_{1}})^{2}(x^{j_{2}}_{n_{2}})^{2}}}. 
\end{equation}

Now notice that if all $x^{j}$a are $\log(n)$-balanced, then the upper bound above reduces to:
\begin{equation}
\mathbb{P}[\mathcal{E}_{i}] \leq 2e^{-\frac{\kappa^{2} n^{2}}{2\log^{4}(n)\chi^{2}(i_{1},i_{1})\sum_{1 \leq n_{1} < n_{2} \leq n}
(2\sigma_{i_{1},i_{1}}(n_{1},n_{2}))^{2}}}. 
\end{equation}

Thus we get:
\begin{equation}
\mathbb{P}[\mathcal{F}_{\kappa}^{i_{1},i_{1},j_{1},j_{2}}] \leq 2\chi(i_{1},i_{1})e^{-\frac{\kappa^{2} n^{2}}{2\log^{4}(n)\chi^{2}(i_{1},i_{1})\sum_{1 \leq n_{1} < n_{2} \leq n}
(2\sigma_{i_{1},i_{1}}(n_{1},n_{2}))^{2}}}. 
\end{equation}

We will now consider dot products of the form $\langle\textbf{s}^{i_{1},j_{1}},\textbf{s}^{i_{2},j_{2}}\rangle$ for $i_{1} \neq i_{2}$. From Lemma \ref{s_lemma} we get:
\begin{align}
\begin{split}
\langle\textbf{s}^{i_{1},j_{1}},\textbf{s}^{i_{2},j_{2}}\rangle &= 
\sum_{n_{1}=1}^{n}\sigma_{i_{1},i_{2}}(n_{1},n_{1})x^{j_{1}}_{n_{1}}x^{j_{2}}_{n_{1}}\\
&\quad +2\sum_{1 \leq n_{1} < n_{2} \leq n} d_{n_{1}}d_{n_{2}}x^{j_{1}}_{n_{1}}x^{j_{2}}_{n_{2}}
\sigma_{i_{1},i_{2}}(n_{1},n_{2})
\end{split}
\end{align}

We will proceed analogously. 
Let us denote by $\mathcal{F}_{\kappa}^{i_{1},i_{2},j_{1},j_{2}}$ the following event:
\begin{equation}
\mathcal{F}_{\kappa}^{i_{1},i_{2},j_{1},j_{2}} = \{|\sum_{n_{1}=1}^{n}\sigma_{i_{1},i_{2}}(n_{1},n_{1})x^{j_{1}}_{n_{1}}x^{j_{2}}_{n_{1}}+
2\sum_{1 \leq n_{1} < n_{2} \leq n} d_{n_{1}}d_{n_{2}}x^{j_{1}}_{n_{1}}x^{j_{2}}_{n_{2}}
\sigma_{i_{1},i_{2}}(n_{1},n_{2}| > \kappa | \mathcal{E}_{balanced}\}.
\end{equation}
Again, we condition on the $\log(n)$-balanceness property.
Note that $\mathcal{F}_{\kappa}^{i_{1},i_{2},j_{1},j_{2}}$ is contained in the event
\begin{equation}
\tilde{\mathcal{F}}_{\kappa}^{i_{1},i_{2},j_{1},j_{2}}=
\{|2\sum_{1 \leq n_{1} < n_{2} \leq n} d_{n_{1}}d_{n_{2}}x^{j_{1}}_{n_{1}}x^{j_{2}}_{n_{2}}
\sigma_{i_{1},i_{2}}(n_{1},n_{2})| > \kappa - \theta | \mathcal{E}_{balanced}\},
\end{equation}
where:
$\theta = |\sum_{n_{1}=1}^{n}\sigma_{i_{1},i_{2}}(n_{1},n_{1})x^{j_{1}}_{n_{1}}x^{j_{2}}_{n_{1}}|$.
Now note that under $\log(n)$-balanceness condition, from the definition of $\tilde{\mu}[\mathcal{P}]$, we get: $\theta \leq  \tilde{\mu}[\mathcal{P}] \frac{\log^{2}(n)}{n}$.

To find an upper bound on $\mathbb{P}[\tilde{\mathcal{F}}_{\kappa}^{i_{1},i_{2},j_{1},j_{2}}]$,
we proceed in a similar way as before and obtain:
\begin{equation}
\mathbb{P}[\tilde{\mathcal{F}}_{\kappa}^{i_{1},i_{2},j_{1},j_{2}}] \leq 2 \chi(i_{1},i_{2})e^{-\frac{(\kappa-\frac{\log^{2}(n)}{n}\tilde{\mu}[\mathcal{P}])^{2}}{2\chi^{2}(i_{1},i_{2})\sum_{1 \leq n_{1} < n_{2} \leq n}
(2\sigma_{i_{1},i_{1}}(n_{1},n_{2}))^{2}}\frac{n^{2}}{\log^{4}(n)}}.
\end{equation}

We have already noted that $\mathbb{P}[\mathcal{F}_{\kappa}^{i_{1},i_{2},j_{1},j_{2}}] \leq 
\mathbb{P}[\tilde{\mathcal{F}}_{\kappa}^{i_{1},i_{2},j_{1},j_{2}}].$
Combining obtained upper bounds on $\mathbb{P}[\mathcal{F}_{\kappa}^{i_{1},i_{2},j_{1},j_{2}}]$ for $i_{1} \neq i_{2}$ and $\mathbb{P}[\mathcal{F}_{\kappa}^{i_{1},i_{1},j_{1},j_{2}}]$, we get the following bound for $j_{1} \neq j_{2}$,
any $1 \leq i_{1} \leq i_{2} \leq m$, fixed $\kappa > 0$ that does not depend on $n$ and $n$ large enough:
\begin{equation}
\mathbb{P}[\mathcal{F}_{\kappa}^{i_{1},i_{2},j_{1},j_{2}}] \leq 2 \chi(i_{1},i_{2})e^{-\frac{(\kappa-\frac{\log^{2}(n)}{n}\tilde{\mu}[\mathcal{P}])^{2}}{2\chi^{2}(i_{1},i_{2})\sum_{1 \leq n_{1} < n_{2} \leq n}
(2\sigma_{i_{1},i_{1}}(n_{1},n_{2}))^{2}}\frac{n^{2}}{\log^{4}(n)}}.
\end{equation}

Now, using the definition of $\chi[\mathcal{P}]$ and $\mu[\mathcal{P}]$, we get: 
\begin{equation}
\mathbb{P}[\mathcal{F}_{\kappa}^{i_{1},i_{2},j_{1},j_{2}}] \leq 2 \chi[\mathcal{P}]e^{-\frac{(\kappa-\frac{\log^{2}(n)}{n}\tilde{\mu}[\mathcal{P}])^{2}}{8\chi^{2}(i_{1},i_{2})\mu^{2}[\mathcal{P}]}\frac{n}{\log^{4}(n)}}.
\end{equation}

Thus taking the union bound over all ${N \choose k}$ possible choices for basis and at most ${k \choose 2} m^{2}$ choices of two different vectors $\textbf{s}^{i_{1},j_{1}}$, $\textbf{s}^{i_{2},j_{2}}$ for a fixed basis, we obtain that conditioned on $\log(n)$-balanceness, the probability of an event $\mathcal{E}_{large}$ that the absolute value of at least one of the dot products under consideration is above $\kappa$ is at most:
\begin{equation}
\mathbb{P}[\mathcal{E}_{large}] \leq k^{2}m^{2}{N \choose k} \chi[\mathcal{P}]e^{-\frac{(\kappa-\frac{\log^2{n}}{n}\tilde{\mu}[\mathcal{P}])^{2}}{8\chi^{2}[\mathcal{P}]\mu^{2}[\mathcal{P}]}\frac{n}{\log^{4}(n)}}.
\end{equation}

Now we focus on finding concentration results for $\|\textbf{s}^{i,j}\|^{2}_{2}$.
Note that from the formula on $\textbf{s}^{i,j}$ given in Lemma \ref{s_lemma} we get:
\begin{equation}
\|\textbf{s}^{i,j}\|^{2}_{2} = 2\sum_{1 \leq n_{1} < n_{2} \leq n} d_{n_{1}}d_{n_{2}}
x^{j}_{n_{1}}x^{j}_{n_{2}}(\sum_{l=1}^{t}p^{i}_{l,n_{1}}p^{i}_{l,n_{2}}) + \sum_{n_{1}=1}^{n}(x^{j}_{n_{1}})^{2}.
\end{equation}

Note that second term on the RHS of the equation above equals $1$.
The expected value of the first term on the RHS is clearly $0$. Thus $\mathbb{E}[\|\textbf{s}^{i,j}\|^{2}_{2}] = 1$.
Note that we want the following:
\begin{equation}
|\|\textbf{s}^{i,j}\|^{2}_{2} - 1| \leq \frac{1}{\log(n)}.
\end{equation}

As before, we will condition now on the choices of $x^{j}$.
Let us denote by $\mathcal{E}^{i,j}_{norm}$ an event:
$\{|\|\textbf{s}^{i,j}\|^{2}_{2} - 1| > \frac{1}{\log(n)}\}$.
We have:
\begin{equation}
\mathbb{P}[\mathcal{E}^{i,j}_{norm}] = 
\mathbb{P}[|2\sum_{1 \leq n_{1} < n_{2} \leq n} d_{n_{1}}d_{n_{2}}
x^{j}_{n_{1}}x^{j}_{n_{2}}(\sum_{l=1}^{t}p^{i}_{l,n_{1}}p^{i}_{l,n_{2}})| > \frac{1}{\log(n)}].
\end{equation}

Thus the following is true:
\begin{equation}
\mathbb{P}[\mathcal{E}^{i,j}_{norm}] = 
\mathbb{P}[|2\sum_{1 \leq n_{1} < n_{2} \leq n} d_{n_{1}}d_{n_{2}}
x^{j}_{n_{1}}x^{j}_{n_{2}}\sigma_{i,i}(n_{1},n_{2})| > \frac{1}{\log(n)}].
\end{equation}

Using the same trick with partitioning the set of random variables under consideration into small number of subsets of independent random variables (to find an upper bound on the probability above) and denoting by $\mathcal{E}_{norm}$
the union of $\mathcal{E}_{norm}^{i,j}$ under all choices of $i,j$ conditioned on $\log(n)$-balanceness, we obtain (by applying a simple union bound) the following bound on $\mathbb{P}[\mathcal{E}_{norm}]$.
\begin{equation}
\mathbb{P}[\mathcal{E}_{norm}] \leq \Gamma.
\end{equation}

From the fact that the choices of $\mathbb{D}_{0}$ and $\mathbb{D}_{1}$ are independent, and applying the union bound, we get:
\begin{equation}
\mathbb{P}[\mathcal{E}^{\kappa}_{dot}] \geq 
\mathbb{P}[\mathcal{E}_{balanced}](1-\mathbb{P}[\mathcal{E}_{large}] - \mathbb{P}[\mathcal{E}_{norm}]).
\end{equation}

The statement of the lemma follows then from Lemma \ref{balanceness_lemma} and upper bounds on $\mathbb{P}[\mathcal{E}_{large}]$ and $\mathbb{P}[\mathcal{E}_{norm}]$ given above.
\end{proof}

We also need the following lemma.

\begin{lemma}
\label{gaussian_lemma}
Let $\textbf{g}=(g_{1},...,g_{t})$ be a $t$-dimensional Gaussian vector with dimensions $g_{i}$ taken independently from $\mathcal{N}(0,1)$. Let $\{\textbf{s}^{1},...,\textbf{s}^{u}\}$ 
be the set of $u$ vectors.
Assume that for some $\kappa, L_{min} > 0$ the following holds: for any two different $\textbf{s}^{i}$, $\textbf{s}^{j}$ their dot product satisfies: $|\langle\textbf{s}^{i},\textbf{s}^{j}\rangle| \leq \kappa$ and furthermore $\|\textbf{s}^{i}\|_{2} \geq L_{min}$ for $i=1,...,u$.
Denote by $\mathcal{E}_{orth}^{c}$ an event that there exists a set of pairwise orthogonal vectors $\{\hat{\textbf{s}}^{i}\}$ for $i=1,...,u$, such that: $\langle\textbf{g}, \textbf{s}^{i}\rangle = \langle\textbf{g},\hat{\textbf{s}}^{i}\rangle + \epsilon$, for $\epsilon$ satisfying:
$|\epsilon| \leq u\frac{c\kappa}{L_{min}}$ and $\|\textbf{s}^{i}\|_{2}=\|\hat{\textbf{s}}^{i}\|_{2}$. Then for $c=c(u)$ large enough (but depending only on $u$) the following holds: 
\begin{equation}
\mathbb{P}[\mathcal{E}_{orth}] \geq 1 - \sqrt{\frac{2u}{\pi}}e^{-\frac{u}{2}}.
\end{equation} 
The probability above is in respect to random choices involving the construction of $\textbf{g}$.
\end{lemma}

\begin{proof}
Take as the set $\{\hat{\textbf{s}}^{i}\}$ the set obtained from $\{\textbf{s}^{1},...,\textbf{s}^{u}\}$ by a Gram-Schmidt orthogonalization process (followed by rescaling procedure to ensure that $\|\textbf{s}^{i}\|_{2}=\|\hat{\textbf{s}}^{i}\|_{2}$). One can check (we leave it to the Reader) that the conditions: $|\langle\textbf{s}^{i},\textbf{s}^{j}\rangle| \leq \kappa$ and: $\|\textbf{s}^{i}\|_{2} \geq L_{min}$ for $i=1,...,u$ imply that for $i=1,...,u$ we have:  $\|\hat{\textbf{s}}^{i} - \textbf{s}^{i}\| \leq \frac{\xi_{gram}(u) \kappa}{L_{min}}$
for $\xi_{gram}(u)$ that depends just on $u$. Thus we will take $c = \xi_{gram}(u)$.
Thus, by Cauchy-Schwarz Inequality, it only remains to show that the probability that a projection $\textbf{g}_{proj}$ of $\textbf{g}$ onto $span(\textbf{s}^{1},...,\textbf{s}^{u})$ has length at most $u$ is at least $1 - \sqrt{\frac{2u}{\pi}}e^{-\frac{u}{2}}$. Note that this projection is a $u$-dimensional Gaussian vector. Thus, by the union bound we get that this probability is at least $1-u\mathbb{P}[|g|^{2} > u]$, where $g$ stands for the $1$-dimensional Gaussian random variable taken from the distribution $\mathcal{N}(0,1)$. Now we use the well-known inequality upper-bounding the tail of the Gaussian random variable $g$, namely:
\begin{equation}
\mathbb{P}[|g| > x] \leq 2\frac{e^{-\frac{x^{2}}{2}}}{\sqrt{2\pi}x}
\end{equation} 
for any $x>0$ and the proof is completed.
\end{proof}

Now we will show the following:

\begin{lemma}
\label{p_lemma}
Fix some $\textbf{v}^{1},...,\textbf{v}^{k} \in \mathbb{R}^{n}$.
Let $\Psi_{unstruct}$ be defined as follows: 
\begin{equation}
\Psi_{unstruct} = \Psi(B^{\textbf{v}^{1},...,\textbf{v}^{k}}(\textbf{r}^{1}_{proj}), ... , B^{\textbf{v}^{m},...,\textbf{v}^{k}}(\textbf{r}^{m}_{proj}))
\end{equation}
(see: theoretical section for the definition of $B$)
and let $\Psi_{struct}$ be defined as follows: 
\begin{equation}
\Psi_{struct} = \Psi(B^{\textbf{v}^{1},...,\textbf{v}^{k}}(\textbf{r}^{1}_{proj}+\textbf{e}^{1}), ... , B^{\textbf{v}^{m},...,\textbf{v}^{k}}(\textbf{r}^{m}_{proj} + \textbf{e}^{m})),
\end{equation}
where $\textbf{r}^{i}_{proj}$ are projections of $n$-dimensional Gaussian vectors $\textbf{r}^{i}$ onto a $mk$-dimensional linear subspace and $\|\textbf{e}^{i}\|_{2} \leq \epsilon$ for some given $\epsilon > 0$.
Then for any $0 \leq \bar{m} \leq m$ and $\lambda > 0$ the following holds. The probability 
of an event $\mathcal{E}_{close}$ that for any choice of vectors $\textbf{e}^{i}$ satisfying a condition $\|\textbf{e}^{i}\|_{2} \leq \epsilon$ we have : $|\Psi_{unstruct} - \Psi_{struct}| \leq \bar{m} \Delta^{\Psi}_{M} + (m-\bar{m})\Delta^{\Psi}_{\lambda}$ is at least 
$\mathbb{P}[\mathcal{E}_{close}] \geq 1 - \sum_{j=\bar{m}+1}^{m} \frac{(p_{\lambda, \epsilon}m)^{j}}{j!}$. The probability that the above is true for all $k$-tuples of vectors of the dataset $\mathcal{X}$ is at least $\mathbb{P}[\mathcal{E}_{close}^{all}] \geq 1 - {N \choose k} \sum_{j=\bar{m}+1}^{m} \frac{(p_{\lambda, \epsilon}m)^{j}}{j!}$
\end{lemma}

\begin{proof}
Fix some $i=1,..,m$. Denote 
\begin{equation}
Y_{i} = \sup_{|\textbf{e}^{i}|_{\infty} \leq \epsilon} |B^{\textbf{v}^{1},...,\textbf{v}^{k}}(\textbf{r}^{i}_{proj}) - B^{\textbf{v}^{1},...,\textbf{v}^{k}}(\textbf{r}^{i}_{proj}+\textbf{e}^{i})|.
\end{equation}
Denote:
\begin{equation}
Z_{i} =
\left\{
	\begin{array}{ll}
		1  & \mbox{if }  |Y_{i}| > \lambda, \\
		0 & \mbox{otherwise.} 
	\end{array}
\right.
\end{equation}

Note that $\{Z_{i}\}$ forms a Bernoulli sequence with the probability of success of every 
$Z_{i}$ at most $p_{\lambda, \epsilon}$ (from the definition of $p_{\lambda, \epsilon}$).
Denote $Z = Z_{1} + ... + Z_{m}$.
Thus for $p_{\lambda, \epsilon} \leq \frac{1}{2}$ the probability that $Z > \bar{m}$ is at most 
$\tilde{p} = \sum_{j=\bar{m}+1}^{m} p_{\lambda, \epsilon}^{j}(1-p_{\lambda, \epsilon})^{m-j}{m \choose j}$.  Thus we have:
$\tilde{p} \leq \sum_{j=\bar{m}+1}^{m} \frac{(p_{\lambda, \epsilon}m)^{j}}{j!}$.
Let $\mathcal{E}_{small}$ be an event that $Z \leq \bar{m}$. We can conclude that:
$\mathbb{P}[\mathcal{E}_{small}] \geq 1 - \sum_{j=\bar{m}+1}^{m} \frac{(p_{\lambda, \epsilon}m)^{j}}{j!}$. But if $\mathcal{E}_{small}$ holds then $\Psi_{unstruct}$ is different from $\Psi_{struct}$ by more than $\lambda$ on at most $\bar{m}$ coordinates (since if $\|\textbf{e}^{i}\|_{2} \leq \epsilon$ then in particular $\|\textbf{e}^{i}\|_{\infty} \leq \epsilon$).
But then, by the definition of $\Delta$, we obtain:
\begin{equation}
|\Psi_{unstruct} - \Psi_{struct}| \leq \bar{m} \Delta^{\Psi}_{M} + (m-\bar{m})\Delta^{\Psi}_{\lambda}.
\end{equation}
Thus we conclude that $\mathcal{E}_{small} \subseteq \mathcal{E}_{close}$. Then we can use the derived lower bound on $\mathcal{E}_{small}$ (to obtain the first statement) and apply union bound (to obtain the second statement) and the proof is completed.
\end{proof}

We have already introduced one standard concentration tool, namely Azuma's Inequality.
Now we will need more refined generalization of it that we state below.

\begin{lemma}(McDiarmid Inequality)
Let $X_{1},...,X_{s}$ be $s$ independent random variables. Assume that $X_{i} \in \mathcal{X}_{i}$ for some measurable sets $\mathcal{X}_{i}$. Suppose that
$h : \prod_{i=1}^{s} A_{i} \rightarrow \mathbb{R}$ satisfies the following.
For each $r \leq s$ and any two sequences $\textbf{x}$ and $\textbf{x}^{\prime}$ that
differ only in the $i^{th}$ coordinate $|h(\textbf{x}) - h(\textbf{x}^{\prime})| \leq \rho_{r}$.
Let $Y = h(X_{1},...,X_{s})$. Then for any $a>0$ we have:
\begin{equation}
\mathbb{P}[|Y-\mathbb{E}[Y]| > a] \leq 2e^{-\frac{2a^{2}}{\sum_{i=1}^{s}\rho_{i}^{2}}}.
\end{equation}
\end{lemma}

We need one more technical auxiliary lemma.
\begin{lemma}
\label{mac_diarmid_lemma}
Let $h(x_{1,1},...,x_{m,k}) = \Psi(\beta(x_{1,1},...,x_{1,k}),...,\beta(x_{m,1},...,x_{m,k}))$
and denote:
\begin{equation}
\rho_{i}^{\Psi,\beta} = \sup_{x_{1,1},...,x_{m,k},y^{\prime}_{1,1},...,y^{\prime}_{m,k}} |h(x_{1,1},...,x_{m,k}) - h(x^{\prime}_{1,1},...,x^{\prime}_{m,k})|,
\end{equation}
where sequences $x_{1,1},...,x_{m,k}$ and $x^{\prime}_{1,1},...,x^{\prime}_{m,k}$ differ only in the $i^{th}$ coordinate.

Let $X_{1,1},...,X_{m,k}$ be $mk$ independent random variables. For a given $a>0$ denote by $\mathcal{E}_{mean}$ the following event: $\{|h(X_{1,1},...,X_{m,k}) - \mathbb{E}[h(X_{1,1},...,X_{m,k})]| \leq a\}$. Then we have:
\begin{equation}
\mathbb{P}[\mathcal{E}_{mean}] \geq 1 - 2e^{-\frac{2a^{a}}{\sum_{i=1}^{mk} (\rho_{i}^{\Psi,\beta})^{2}}}.
\end{equation}
\end{lemma}
\begin{proof}
Follows immediately from McDiarmid Inequality.
\end{proof}

We are ready to prove Theorem \ref{general_theorem}.
\begin{proof}
In the proof we denote the rows of the structured matrix $\textbf{AD}_{1}$ as $\textbf{w}^{1},...,\textbf{w}^{m}$. Without loss of generality we will therefore assume that the preprocessing step consists just of the linear mapping $\textbf{HD}_{0}$ (thus we moved matrix $\textbf{D}_{1}$ to the second step).
Denote the computed structured version of $\Lambda_{f}(\textbf{v}^{1},...,\textbf{v}^{k})$
as $\Lambda_{f}^{struct}(\textbf{v}^{1},...,\textbf{v}^{k})$.
Note that the expression: $err= |\Lambda_{f}(\textbf{v}^{1},...,\textbf{v}^{k}) - \Lambda_{f}^{struct}(\textbf{v}^{1},...,\textbf{v}^{k})|$ can be upper-bounded by:
\begin{align}
\begin{split}
err &\leq |\mathbb{E}[\Psi(\beta(y_{1,1},...,y_{1,k}),...,\beta(y_{m,1},...,y_{m,k}))]\\
&\qquad -\Psi(\beta(y_{1,1},...,y_{1,k}),...,\beta(y_{m,1},...,y_{m,k}))|\\
&\quad +|\Psi(\beta(y_{1,1},...,y_{1,k}),...,\beta(y_{m,1},...,y_{m,k}))\\
&\qquad -\Psi(\beta(\tilde{y}_{1,1},...,\tilde{y}_{1,k}),...,\beta(\tilde{y}_{m,1},...,\tilde{y}_{m,k}))|,\\
\end{split}
\end{align}

where $y_{i,j} = \langle\textbf{r}^{i},\textbf{v}^{j}\rangle$ and $\tilde{y}_{i,j} = \langle\textbf{w}^{i},\textbf{v}^{j}\rangle$.
Note that by Lemma \ref{mac_diarmid_lemma} and the union bound the probability of an event $\mathcal{E}_{0}$ that for some $k$-tuple of vectors $\textbf{v}^{1},...,\textbf{v}^{k}$ the first expression on the RHS of the inequality above is greater than some fixed $K>0$ is at most 
$\mathbb{P}[\mathcal{E}_{0}] \leq {N \choose k} 2e^{-\frac{2K^{2}}
{\sum_{i=1}^{mk}(\rho_{i}^{\Psi, \beta})^{2}}}$.
Note that the second expression on the RHS of the inequality above can be rewritten as:
\begin{align}
\begin{split}
\label{important_equation}
err_{bound} &= |\Psi(B^{\textbf{v}^{1},...,\textbf{v}^{k}}(\textbf{r}^{1}_{proj}),...,B^{\textbf{v}^{1},...,\textbf{v}^{k}}(\textbf{r}^{m}_{proj})) \\
&\quad -\Psi(B^{\textbf{v}^{1},...,\textbf{v}^{k}}(\textbf{r}^{1}_{proj} + \textbf{e}^{i}),...,B^{\textbf{v}^{1},...,\textbf{v}^{k}}(\textbf{r}^{m}_{proj} + \textbf{e}^{m}))|,
\end{split}
\end{align}
where $r^{i}_{proj}$ is a $k$-dimensional gaussian vector of the form 
\begin{equation}
\textbf{r}^{i}_{proj} = (\langle\textbf{g}, \tilde{\textbf{s}}^{i,1}\rangle,...,\langle\textbf{g}, \tilde{\textbf{s}}^{i,k}\rangle),
\end{equation}
and 
\begin{equation}
\epsilon^{i} = (\langle\textbf{g}, \textbf{s}^{i,1} - \tilde{\textbf{s}}^{i,1}\rangle,...,\langle\textbf{g}, \textbf{s}^{i,k} - \tilde{\textbf{s}}^{i,k}\rangle).
\end{equation}
In the equations above $\textbf{g}$ is a Gaussian vector from the definition of the $\mathcal{P}$-model and $\textbf{s}^{i,j}$ and $\tilde{\textbf{s}}^{i,j}$ are as in Lemma \ref{gaussian_lemma}. Here we use the fact that projections of a Gaussian vector onto orthogonal directions are independent Gaussian random variables of the same distribution.
Now define an event: $\mathcal{E}_{good}$ as follows. For any $k$-tuple of independent vectors $\textbf{v}^{1},...,\textbf{v}^{k}$ from $\mathcal{X}$ we have: $err_{bound} \leq \bar{m}\Delta^{\Psi}_{M} + (m-\bar{m})\Delta^{\Psi}_{\lambda}$.
Note that clearly $\mathcal{E}_{good} \subseteq \mathcal{E}^{all}_{close} \cup \mathcal{E}_{orth} \cup \mathcal{E}^{\kappa}_{dot}$ for $\kappa = \frac{\epsilon L_{min}}{c}$, where $L_{min} = \sqrt{1-1\frac{1}{\log(n)}}$ and $c$ is as in Lemma \ref{gaussian_lemma}. 
We can conclude, by the union bound argument, that the probability that the structured computation produces a result that differs from the exact value of function $\Lambda_{f}(\textbf{v}^{1},...,\textbf{v}^{k})$
for at least one $k$-tuple of independent vectors from the dataset $\mathcal{X}$ by more than
$K + \bar{m}\Delta^{\Psi}_{M} + (m-\bar{m})\Delta^{\Psi}_{\lambda}$ is at most:
\begin{equation}
\mathbb{P}[\mathcal{E}_{wrong}] \leq \mathbb{P}[\mathcal{E}_{0}] + 
(1-\mathbb{P}[\mathcal{E}_{close}^{all}])+(1-\mathbb{P}[\mathcal{E}_{orth}])+(1-\mathbb{P}[\mathcal{E}_{dot}^{\kappa}]).
\end{equation}
Taking $n$ large enough and using derived earlier bounds on
$\mathbb{P}[\mathcal{E}_{0}]$ , $\mathbb{P}[\mathcal{E}_{close}^{all}]$, $\mathbb{P}[\mathcal{E}_{orth}]$ and $\mathbb{P}[\mathcal{E}_{dot}^{\kappa}]$, we complete the proof
for the case $M < \infty$.

Now let us assume that $M = \infty$. We will use use equation \ref{important_equation}.
Let us fix some $\textbf{e}^{1},...,\textbf{e}^{m}$ such that $|\textbf{e}^{i}|_{\infty} \leq \epsilon$. Denote $Y_{i}^{\textbf{e}^{i}} = B^{\textbf{v}^{1},...,\textbf{v}^{k}}(\textbf{r}^{i}_{proj} + \textbf{e}^{i})$. Note that in this case $\Psi$ is just averaging thus let us define:
$Y_{E} = \frac{Y_{1}^{\textbf{e}^{1}} + ... + Y_{m}^{\textbf{e}^{m}}}{m}$, where: 
$E = \{\textbf{e}^{1},...,\textbf{e}^{m}\}$.
We want to upper-bound the probability: $p_{large}$ defined as:
\begin{equation}
\label{second_important_equation}
p_{large} = \mathbb{P}[\exists_{E} : |Y_{E}-\mathbb{E}[Y_{E}]| > a]$
for $a = m^{-\frac{1}{2\alpha}}.
\end{equation}
The probability $p_{large}$, from the definition of the Legendre symbol,  is at most:
\begin{equation}
p_{large} \leq \int_{x=a}^{\infty} (-\sup_{\textbf{e}^{i}} e^{-m\mathcal{L}_{\textbf{e}^{i}}(x)})^{\prime}_{|x} dx + 
\int_{x=-\infty}^{-a} (-\sup_{\textbf{e}^{i}} e^{-m\mathcal{L}_{\textbf{e}^{i}}(x)})^{\prime}_{|x} dx
\end{equation}
where $\mathcal{L}_{\textbf{e}^{i}}(x)$ stands for the Legendre symbol of a random variable 
$Y_{i}^{\textbf{e}^{i}} - \mathbb{E}[Y_{i}^{\textbf{e}^{i}}]$, evaluated at $x$.

Thus we have:
\begin{equation}
p_{large} \leq m\int_{x=a}^{\infty} \sup_{\textbf{e}^{i}} (e^{-m\mathcal{L}_{\textbf{e}^{i}}(x)} \mathcal{L}_{\textbf{e}^{i}}^{\prime}(x)) dx + 
m\int_{x=-\infty}^{-a} \sup_{\textbf{e}^{i}} (e^{-m\mathcal{L}_{\textbf{e}^{i}}(x)} \mathcal{L}_{\textbf{e}^{i}}^{\prime}(x)) dx
\end{equation}

Now we can use our assumptions regarding $\mathcal{L}_{\textbf{e}^{i}}$ and we get:
\begin{equation}
p_{large} \leq mc_{2}\int_{x=a}^{\infty} e^{-mc_{1}x^{\alpha}}x^{\beta}dx + 
mc_{2}\int_{x=-\infty}^{-a} e^{-mc_{1}|x|^{\alpha}}|x|^{\beta}dx  
\end{equation}
for some constants $c_{1},c_{2}>0$.
Thus it is easy to see that for $m$ large enough (but independent on $n$) we have:
\begin{equation}
p_{large} \leq mc_{2}\int_{x=a}^{\infty} e^{-\frac{mc_{1}x^{\alpha}}{2}}dx + 
mc_{2}\int_{x=-\infty}^{-a} e^{-\frac{mc_{1}|x|^{\alpha}}{2}}dx.
\end{equation}

Thus under our choice of $a$ we get:
\begin{equation}
p_{large} = O(e^{-\Omega(\sqrt{m})}).
\end{equation}

Thus, for a fixed set of $k$ vectors $\textbf{v}^{1},...,\textbf{v}^{k}$ we get from the definition of $\tilde{\beta}_{\epsilon}$ that the probability
that there exists $E$ such that $|Y_{E}-\Lambda_{f}(\textbf{v}^{1},...,\textbf{v}^{k})| > a + \tilde{\beta}_{\epsilon}$ is at most $p_{large}$. Thus we can use our analysis for the case and $M < \infty$ (to obtain terms on the probabilities regarding structured properties of the linear projection matrix), take the union bound over $k$-tuples and the proof is completed also for the case $M = \infty$.

\section{Computation of $p_{0, \epsilon}$ for the angular case}

Note that in the main body of the paper we claimed that for the angular similarity the following is true: $p_{0, \epsilon} \leq \frac{2\sqrt{2}m\epsilon}{\pi} + \frac{2}{\pi m^{2}}$.
We will prove it now.

\begin{proof}
Fix two vectors $\textbf{v}^{1}$ and $\textbf{v}^{2}$ and let $\textbf{g}_{proj}$ be a projection of a Gaussian vector on $span(\textbf{v}^{1},\textbf{v}^{2})$.
Note that the probability that the $L_{2}$ norm of $\textbf{g}_{proj}$ is at most $\frac{1}{m}$
is at most: $(\mathbb{P}[g^{2} \leq \frac{1}{m^{2}}])^{2}$, where $g \sim \mathcal{N}(0,1)$.
Thus we have:
\begin{equation}
\mathbb{P}[\|\textbf{g}_{proj}\|_{2} \leq \frac{1}{m}] \leq (\frac{1}{\sqrt{2\pi}}\frac{1}{m} \cdot 2)^{2} \leq \frac{2}{\pi m^{2}}.
\end{equation}
Now note that by adding to a vector $\textbf{v} \in span(\textbf{v}^{1}, \textbf{v}^{2})$ of $L_{2}$-norm $\|\textbf{v}\|_{2} > \frac{1}{m}$ a ``perturbation vector'' 
$\textbf{e} \in span(\textbf{v}^{1}, \textbf{v}^{2})$ such that: $\|\textbf{e}\|_{\infty} \leq \epsilon$ one changes an angle between $\textbf{v}$ and $\textbf{v}^{1}$ (and thus also $\textbf{v}^{2}$) by at most $\theta_{\epsilon}$, where: 
\begin{equation}
\tan(\theta_{\epsilon}) \leq 
\frac{\sqrt{\epsilon^{2}+\epsilon^{2}}}{\|v\|_{2}} \leq \sqrt{2}\epsilon m.
\end{equation}
Now we use the Taylor expansion of $\tan(x) = x + \frac{x^{3}}{3} + \frac{2x^{5}}{15} + ...$
and obtain: 
\begin{equation}
\theta_{\epsilon} \leq \sqrt{2}\epsilon m.
\end{equation}
Conclude that function $f$ can change its value by perturbating by vector $\textbf{e}$
only if $\textbf{g}_{proj}$ resides in the union of two $2$-dimensional coins, each of angle at most $2\theta_{\epsilon}$. Using the fact that $\textbf{g}_{proj}$, as a Gaussian vector, is isotropic
and applying simple union bound, we thus get:
\begin{equation}
p_{0, \epsilon} \leq \frac{4 \cdot \sqrt{2} \epsilon m}{2 \pi} + \frac{2}{\pi m^{2}} 
\end{equation}
\end{proof}
and the proof is completed.
\end{proof}

\section{Proofs of Theorem \ref{cor1} and Theorem \ref{cor2}}

We are ready to prove Theorem \ref{cor1} and Theorem \ref{cor2}. We start with Theorem \ref{cor1}.
\begin{proof}
We will apply Theorem \ref{general_theorem}. Note that one can easily see that $\rho_{i}^{\Psi, \beta} \leq \frac{1}{m}$ and $\Delta^{\Psi}_{M} \leq \frac{1}{m}$. We take $\bar{m} = \frac{m}{\log(m)}$, $K = \frac{1}{m^{\tau}}$ and $\epsilon = \frac{\pi}{2\sqrt{2}}\frac{1}{m\log^{2}(m)}$.
We have: 
\begin{equation}
p_{1} = \sum_{j=\bar{m}+1}^{m} \frac{(p_{0, \epsilon}m)^{j}}{j!} \leq 
\sum_{j=\bar{m}+1}^{m} \frac{(\frac{2m}{\log^{2}(m)})^{j}}{\sqrt{2 \pi j} (\frac{j}{e})^{j}},
\end{equation}
where the last inequality comes from the bound obtained in the previous section and
Stirling's formula: $j! \geq \sqrt{2 \pi j} (\frac{j}{e})^{j}$. Thus we obtain:
\begin{equation}
p_{1} \leq m(\frac{2e}{\log(m)})^{\frac{m}{\log(m)}}.
\end{equation}
Denote $p_{2} = 2e^{-\frac{2K^{2}}{\sum_{i=1}^{mk}(\rho_{i}^{\Psi,\beta})^{2}}}$.
Note that under our choice of $K$ we get: 
\begin{equation}
p_{2} \leq 2e^{-m(1-2\tau)}.
\end{equation}
One can easily notice that for structured matrices considered in the statement of the theorem
the required condition on $\tilde{\mu}[\mathcal{P}]$ is satisfied. Furthermore, $\chi[\mathcal{P}] \leq 3$ since in the corresponding coherence graph every vertex has degree at most $2$ (we use a well known result that a graph of maximum degree $d_{max}$ can be colored by $d_{max}+1$ colors). Finally, one can easily notice that $\mu[\mathcal{P}] = O(1)$.
Now it suffices to note that under our choice of parameters the value of the expression $err$ from the statement of Theorem \ref{general_theorem} is: $K + \bar{m}\Delta^{\Psi}_{M} \leq m^{-\tau} + \frac{1}{\log(m)}$. We can then apply Theorem \ref{general_theorem} and the result follows.
\end{proof}

Now we prove Theorem \ref{cor2}.

\begin{proof}
We proceed as in the proof of Theorem \ref{cor1}.
This time we use the observation from the paragraph regarding general kernels 
about values of $\lambda$ for which $p_{\lambda, \epsilon} = 0$.
We take $\bar{m} = 0$ and $\lambda = 2f_{max}\rho + \rho^{2}$.
Other parameters are chosen in the same way as in the proof above. Note that
$\Delta^{\Psi}_{\lambda} \leq \frac{\lambda}{m}$
and $p_{2} \leq e^{-(\frac{m}{f^{2}_{max}})(1-2\tau)}$.
The expression on error $err$ reduces to $K + m\Delta^{\Psi}_{\lambda} = K + \lambda$.
The result then follows from Theorem \ref{general_theorem}.
\end{proof}

\end{document}